\documentclass[journal]{IEEEtran}

\usepackage{graphics} 
\usepackage{epsfig} 
\usepackage{mathptmx} 
\usepackage{times} 
\usepackage{amsmath} 
\usepackage{amssymb}  
\usepackage{amsthm}
\usepackage{bm}
\usepackage{mathrsfs}
\usepackage{xcolor}
\usepackage{cite}
\usepackage{threeparttable}
\usepackage{multirow}
\usepackage{bigdelim}
\usepackage{algorithm}
\usepackage{algorithmicx}
\usepackage{algpseudocode}
\usepackage{graphicx}
\usepackage{subfigure}
\usepackage{comment}
\usepackage{array}
\usepackage{hyperref}
\usepackage{galois}
\usepackage{newtxmath}
\usepackage{enumitem}
\usepackage{makecell}

\newtheorem{theorem}{Theorem}

\newtheorem{definition}[theorem]{Definition}
\newtheorem{lemma}[theorem]{Lemma}
\newtheorem{remark}[theorem]{Remark}

\floatname{algorithm}{Algorithm}

\title{
Sparse Waypoint Validity Checking for Self-Entanglement-Free Tethered Path Planning
}

\author{Tong Yang, \emph{Member, IEEE}, Jiangpin Liu, \\
Yue Wang$^*$, \emph{Member, IEEE}, and Rong Xiong, \emph{Member, IEEE}
\thanks{This work was supported by the National Key R\&D Program of China under Grant 2021ZD0114500. 
\textit{(Corresponding author: Yue Wang and Rong Xiong.)}}
\thanks{Tong Yang, Jiangpin Liu, Yue Wang, and Rong Xiong are with the State Key 
Laboratory of Industrial Control and Technology, Zhejiang University, P.R. China. }
}

\begin{document}

\maketitle

\begin{abstract}
A novel mechanism to derive self-entanglement-free (SEF) path for tethered differential-driven robots is proposed in this work. 
The problem is tailored to the deployment of tethered differential-driven robots in situations where an omni-directional tether re-tractor is not available. 
This scenario is frequently encountered when it is impractical to concurrently equip an omni-directional tether retracting mechanism with other geometrically intricate devices, such as a manipulator, which is notably relevant in applications like disaster recovery, spatial exploration, etc. 
Without specific attention to the spatial relation between the shape of the tether and the pose of the mobile unit, the issue of self-entanglement arises when the robot moves, resulting in unsafe robot movements and the risk of damaging the tether. 
In this paper, the SEF constraint is first formulated as the boundedness of a relative angle function which characterises the angular difference between the tether stretching direction and the robot's heading direction. 
Then, a constrained searching-based path planning algorithm is proposed which produces a path that is sub-optimal whilst ensuring the avoidance of tether self-entanglement. 
Finally, the algorithmic efficiency of the proposed path planner is further enhanced by proving the conditioned sparsity of the primitive path validity checking module. 
The effectiveness of the proposed algorithm is assessed through case studies, comparing its performance against untethered differential-driven planners in challenging planning scenarios. 
To further evaluate the enhanced algorithmic efficiency, a comparative analysis is conducted between the normal node expansion module and the improved node expansion module which incorporates sparse waypoint validity checking. 
Real-world tests are also conducted to validate the algorithm's performance. 
An open-source implementation has also made available for the benefit of the robotics community. 
\end{abstract}


\section{Introduction}

\begin{figure}[t]
\centering
\subfigure[path with self-entanglement]{
\includegraphics[width=0.22\textwidth]{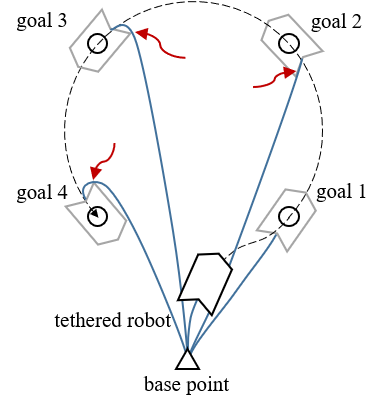}\label{fig:fig1a}
}
\subfigure[self-entanglement-free path]{
\includegraphics[width=0.22\textwidth]{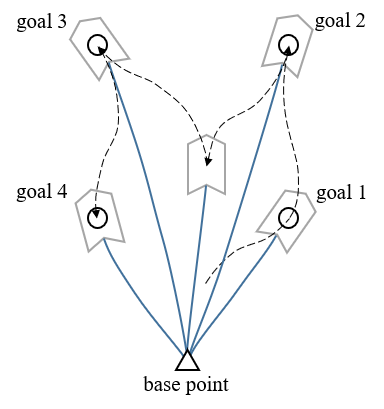}\label{fig:fig1b}
}
\subfigure[path with self-entanglement]{
\includegraphics[width=0.22\textwidth]{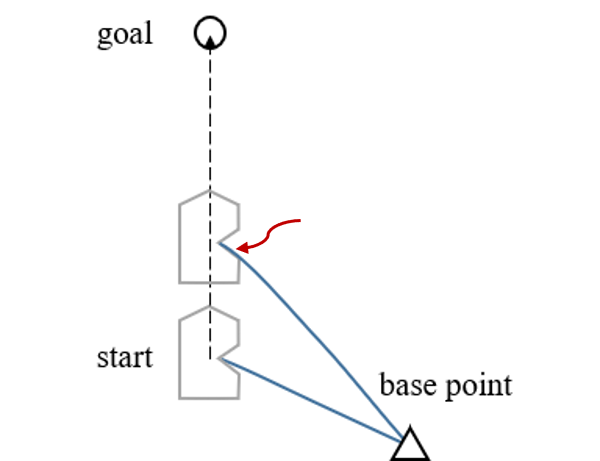}\label{fig:fig1c}
}
\subfigure[self-entanglement-free path]{
\includegraphics[width=0.22\textwidth]{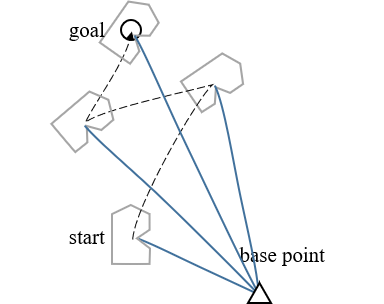}\label{fig:fig1d}
}
\caption{
Examples of various tethered robot paths to visit the predefined goals, depicted as small circles. 
The base point anchoring the tether is depicted as a triangle. 
For simplicity, no obstacle is presented, and the maximum tether length constraint is ignored. 
The mobile unit (depicted as a polygon) is differential-driven, with the tether (blue curve) being hooked (a)(b) to the back and (c)(d) to the right. 
(a) In this case, the circular path is valid for an untethered robot but not suitable for a tethered robot. 
The tether would contact the rear wheel after the robot visits goal $2$, leading to self-entanglement. 
(b) To avoid self-entanglement, the robot, after visiting goal $2$, executes a backward motion to an intermediate pose that allows for the subsequent self-entanglement-free path. 
(c) In this case, for a robot whose tether extends to the right, even a straight forward path would result in self-entanglement. 
The admissible tether retracting orientation at the tether-robot anchoring point is depicted by the notch on the robot footprint.  
(d) A self-entanglement-free resulting path for the case described in (c). 
}\label{fig:fig1}
\end{figure}

Tethered robots, which are mobile robots equipped with a tether anchoring to a fixed base point, have natural advantages in maintaining stable communication links and ensuring continuous power and material supplies. 
This makes them highly suitable for executing energy-intensive tasks and operating in environments where wireless communication is unreliable or unfeasible. 
This is often the case such as sewer pipe inspection~\cite{Nassiraei2007Concept}, highway maintenance~\cite{Hong1997Tethered}, coverage tasks~\cite{Shnaps2014Online}~\cite{Mechsy2017Novel}~\cite{Sharma20192}, disaster recovery missions~\cite{Pratt2008Use}, mountain climbing tasks~\cite{Abad2011Motion}~\cite{Tanner2013Online}, and exploration tasks~\cite{Shapovalov2020Exploration}. 
However, in most of these scenarios, the deployed mobile robots were not originally designed for tethered applications. 
As a result, they lack an omni-directional tether-robot anchoring mechanism. 
In this case, a phenomenon referred to as \textit{self-entanglement} arises which impacts the safety of the tethered robot movement: 
If the mobile unit executes cyclic rotations, then the tether would entangle with the mobile unit. 
The self-entanglement problem fundamentally stems from the inappropriate relative angle between the robot's heading direction and the tether stretching direction. 
See Fig.~\ref{fig:fig1} for the visual illustration of how the tether orientation may influence the movement of a tethered differential-driven robot with the aim to avoid self-entanglement. 
Furthermore, due to the force of gravity, physically the tether will always sag. 
Consequently, self-entanglement may further translate to the wheels rolling over the tether, resulting in unexpected tethered robot configurations (with a self-crossing tether shape) and causing damage to the tether structure. 
Addressing this specialised path planning problem for tethered robots, referred to as the \textit{Self-Entanglement-Free Tethered Path Planning} (SEFTPP) task, is the objective of this paper. 

Given a self-entanglement-free (SEF) tethered robot configuration, the admissible robot motion must be conditioned by not only established modules like collision avoidance, differential-driven robot kinematic constraints, and maximum tether length constraint, but also by the tether stretching direction. 
The necessity of maintaining the SEF property makes the SEF path solution non-intuitive even for seemingly straightforward tasks, as illustrated by the SEF paths in Fig.~\ref{fig:fig1b} and Fig.~\ref{fig:fig1d}. 
It is noteworthy that all existing works in the field of tethered robot planning (to be recounted in Section~\ref{section_related_works}) have assumed scenarios involving either particle robots or omni-directional (2D) robots. 
In these cases, the self-entanglement problem was safely ignored. 
However, such assumptions are far from practical when dealing with real-world challenging applications. 
And one can expect that, restricting the admissible robot heading orientation will significantly constrain the range of valid movements, making the SEFTPP problem non-trivial and more complex compared to untethered planning problems. 

This work advocates for reporting the first solution to the guaranteed self-entanglement-free path for a tethered differential-driven robot. 
The proposed algorithm departs from all existing tethered robot planners. 
It explicitly takes into account the orientation difference between the robot's heading direction and the tether stretching direction, which naturally motivates a constrained path planner to solve the SEFTPP problem. 
The contributions of this paper can be summarised as: 
\begin{enumerate}
\item The modelling of the SEFTPP problem into a constrained path planning problem, with explicit consideration of the relative angle between the robot's heading direction and the tether stretching direction. 
\item A constrained searching-based SEFTPP solution for tethered differential-driven mobile robots. 
\item The proofs demonstrating that under specific conditions during the node expansion phase of the searching-based path planner, the validity of the endpoint configurations of a primitive path ensure the validity of all intermediate waypoint configurations. 
As a result, there is no necessity to explicitly construct any waypoint configuration, nor is there a need to check their validity individually. 
This property is denoted as the \textit{sparsity} of the waypoint configuration validity checking. 
\item The open-sourcing~\footnote{\url{https://github.com/ZJUTongYang/seftpp}} of the algorithm. 
\end{enumerate}

The remainder of this paper is organised as follows. 
Section~\ref{section_related_works} reviews and contextualises the problem within the existing literature. 
Section~\ref{section_problem} formally models the SEFTPP problem. 
Section~\ref{section_algorithm} delves into details to describe the proposed solver to generate the SEFTPP solution.
The sparsity of waypoint configuration validity checking during node expansion is discussed in Section~\ref{section:sparsity}.
Experimental illustrations and comparisons are collected in Section~\ref{section_exp}, with final concluding remarks gathered in Section~\ref{section_conclusion}. 

\section{Related Works}\label{section_related_works}

The \textit{tethered path planning} (TPP) task has been intensively investigated in the past decades. 
Numerous applications of tethered robots have been explored in various domains~\cite{Hong1997Tethered}~\cite{Nassiraei2007Concept}~\cite{Pratt2008Use}~\cite{Abad2011Motion}~\cite{Tanner2013Online}~\cite{Shnaps2014Online}~\cite{Shapovalov2020Exploration}. 
Main concentrations have been paid on constructing the shortest tethered robot path complying with the maximum tether length constraint. 
It has been observed~\cite{Bhattacharya2012Topological} that the tether states will be non-homotopic if the robot reaches the same goal following paths in different topological routes, and the shortest path for an untethered robot is in all likelihood untrackable by a tethered robot. 

Early work confined their scope to polygonal environments~\cite{Teshnizi2014Computing}~\cite{Salzman2015Optimal} so that the algorithmic complexity can be calculated~\cite{Xavier1999Shortest}~\cite{Brass2015Shortest} as a polynomial of the number of straight segments in the initial tether state and the number of obstacle vertices. 
In recent years, a notable advancement is the ability to distinguish between different tethered robot configurations with the same mobile unit pose. 
This is achieved by calculating the homotopy classes of the tether states. 
And the most popular solutions to find the tethered robot path are based on the path-finding within the pre-calculated set of all valid configurations, referred to as the homotopy augmented graph~\cite{Bhattacharya2012Topological}. 
Later, with the utilisation of a locally obstacle-free shortest path planner, the deformation of robot tether becomes easier to estimate. 
This improvement enhances the efficiency of the pre-calculation process~\cite{Kim2014Path}~\cite{Kim2015Path}.
Moreover,~\cite{Teshnizi2014Computing} pre-computed the reachable cell for tethered robots, which improves the efficiency of querying process during the planning phase. 

It is important to note that all the previously mentioned works have focused on the scenarios involving either particle robots or omni-directional robots. 
In these studies, the primary consideration was the tether entanglement with environmental obstacles, ignoring the self-entanglement which is the main motivation of this work. 
In this paper, special treatment of the tether self-entanglement will be incorporated into the proposed algorithm, leading to the generation of robot paths that are guaranteed to be free from self-entanglement, filling in the gap between simulated tethered path planners and real-world execution. 

\section{Problem Modelling}\label{section_problem}

\begin{figure}[t]
\centering
\includegraphics[width=0.35\textwidth]{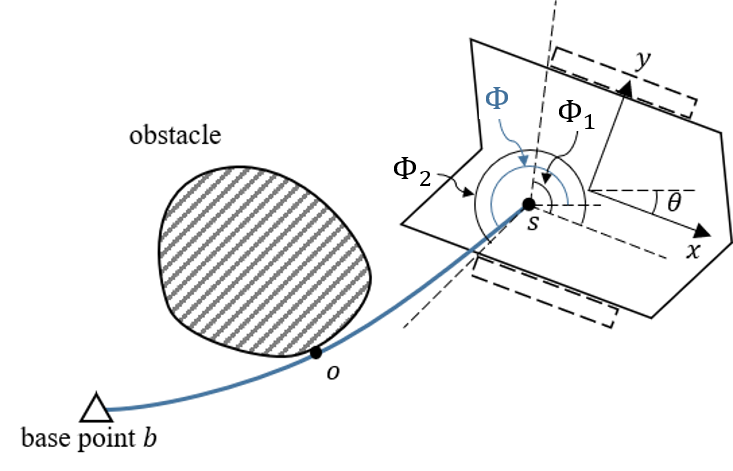}
\caption{Illustration of definitions and notations. }\label{fig:robot}
\end{figure}

This section introduces notations for the environmental settings and the kinematics of the tethered differential-driven robot, along with the formulation of the self-entanglement-free tethered path planning (SEFTPP) problem. 

\subsection{Definitions and Robot Kinematics}
Let $M\subset \mathbb{R}^2$ represent the environment where the tethered robot operates. 
The footprint of the robot's mobile unit is assumed as a polygon. 
The anchoring point of the tether on the robot, defined within the robot's egocentric frame, is denoted as $s$. 
The other endpoint of the tether is anchored at a fixed base point in the environment, which is denoted as $b$. 
The tether is assumed to be taut, allowing its shape to be characterised by a sequence of contact points between the tether and the vertices of environmental obstacles. 
The last tether-obstacle contact point is denoted as $o$. 
Please refer to Fig.~\ref{fig:robot} for a visual illustration of these notations. 

\begin{definition}
(Configuration) The configuration, denoted as $c$, of a tethered differential-driven robot consists of two components: the pose of the differential-driven robot and the shape of the tether. 
It is represented as:
\begin{equation}
c = \{x, y, \theta, O\}
\end{equation}
where $x$, $y$, and $\theta$ represent the $SE(2)$ pose of the mobile unit of the tethered robot, and $O$ maintains a record of the tether-obstacle contact points, from the base point $b$ to the last tether-obstacle contact point $o$.  
\end{definition}

\begin{definition}
(Relative Angle) Under the assumption that the tether remains taut, the direction of tether retraction, denoted as $\phi$, is estimated as:
\begin{equation}
\phi = \frac{\vec{so}}{\parallel \vec{so}\parallel}
\end{equation} 
then $\Phi$ is defined as the relative angle between the tether retracting direction and the robot's heading direction, 
\begin{equation}
\Phi = \phi - \theta
\end{equation}
\end{definition}

Given the geometric structure of the tethered robot, without an omni-directional tether-robot anchoring mechanism, the range of admissible relative angles, denoted as $\Phi$, is not $[0, 2\pi)$ but falls within an interval denoted as $[\Phi_1, \Phi_2]$. 
This forms the basis of the self-entanglement-free property, as discussed in the next subsection. 

\subsection{Self-Entanglement-Free Tethered Path Planning}
Given the initial configuration of the robot $c_s = (x_s, y_s, \theta_s, O_s)$ and the goal location $p_{\rm goal}=(x_g, y_g)$, the solution to the SEFTPP problem is a curve of mobile unit waypoints, represented as:
\begin{equation}
\begin{aligned}
\alpha: [0, 1]\rightarrow &\mathbb{R}^2,\ t\mapsto (x(t), y(t))\in \mathbb{R}^2, t\in [0, 1]\\
&s.t.\ x(0) = x_s, y(0) = y_s, x(1) = x_g, y(1) = y_g
\end{aligned}
\end{equation}
where each \textit{induced} configuration
\footnote{The robot's heading direction $\theta(t)$ is the tangent of the path. 
The tether state $O$ is calculated by the locally obstacle-free tautening of the concatenation of the robot path and the trace of the tether-robot anchoring point $s$. }
$(x(t), y(t), \theta(t), O(t))$ is \textit{valid}, i.e., subject to the following conditions:
\begin{enumerate}
\item \textbf{(Collision-free):} The mobile unit at $(x(t), y(t), \theta(t))$ remains free of collisions. 
\item \textbf{(Tether-length-admissible, TLA):} The length of tether remains shorter than the maximum allowable tether length. 
\item \textbf{(Non-selfcrossing, NS):} The robot is prohibited from traversing across the tether. 
\item \textbf{(Self-entanglement-free, SEF):} The SEF condition is defined by the boundedness of the \textit{relative angle function} which is defined as
\begin{equation}\label{eqn:Phi}
\Phi(t) = \arctan\left( \frac{o_y(t) - \tilde{y}(t)}{o_x(t)-\tilde{x}(t)} \right) - \theta(t),\ t\in [0, 1]
\end{equation}
where $(\tilde{x}(t), \tilde{y}(t))$ is the position of $s$ and $(o_x(t), o_y(t))$ is the position of $o$. 
The SEF condition mandates that this function is bounded by the admissible interval: 
\begin{equation}\label{equ:sef}
\Phi_1 \leq {\rm wrapToPi}(\Phi(t)) \leq \Phi_2, \forall t\in [0, 1]
\end{equation}
\end{enumerate}

It should be noted that SEFTPP has been framed as a path planning problem with multiple potential goals. 
In this context, the goal is represented not as a single configuration but as a 2D location in the environment. 
There can be multiple ``goal configurations" that correspond to a given ``goal location". 
The path planning task is deemed finished as soon as the robot reaches any of these goal configurations. 
This setting is justified because, regarding the target configuration of the tethered robot, the admissibility of the robot's final heading direction $\theta$ (if assigned) is inherently determined by the retracting direction of the tether. 
However, the final shape of the tether, which is further constrained by the tether-length-admissible property, cannot be intuitively determined based solely on human empirical knowledge. 
Therefore, it is a reasonable practice for users to provide a goal location without fully specifying a goal configuration. 

\begin{algorithm}[t]
    \caption{Self-Entanglement-Free Tethered Path Planner}\label{alg:main_planner}
    \begin{algorithmic}[1]
        \Require Map $M$, Initial configuration $c_0$, Goal location $p_{\rm goal}$, Base point $b$, Maximum tether length $L$, Resolution $x_{\rm res}, y_{\rm res}, \theta_{\rm res}$
        \Ensure Resultant path $R$ 
\State \% Initialize Data Structure
\State $\{\zeta_i\}$ = getRepresentativePointOfObstacles($M$)
\State $n_0$ = initialiseNode($c_0$)\label{line:init_start}
\State $[x_d, y_d, \theta_d]$ = getIndex($n_0, x_{\rm res}, y_{\rm res}, \theta_{\rm res}$)
\State $V(x_d, y_d, \theta_d).{\rm push}(n_0)$ \% grid-based discretisation
\State $Q.{\rm push}(n_0)$ \% The priority queue\label{line:init_end}
\While{$\sim$isempty($Q$)}
\State $n_{\rm cur}$ = $Q.{\rm pop}()$
\If{isGoalReached($n_{\rm cur}$, $p_{\rm goal}$)}
\State $R$ = [tracePath($V$, $n_{\rm cur}$);$p_{\rm goal}$]
\State \textbf{return} $R$
\EndIf
\State $N_{\rm succ}$ = nodeExpansion($n_{\rm cur}$)\label{line:getSuccessor}
\For{each element $n_{\rm succ}$ in $N_{\rm succ}$}
\State $[x_{\rm ds}, y_{\rm ds}, \theta_{\rm ds}]$ = getIndex($n_{\rm succ}, x_{\rm res}, y_{\rm res}, \theta_{\rm res}$)
\State $n_{\rm homo}$ = findHomoNode($n_{\rm succ}$,$V(x_{\rm ds}$,$y_{\rm ds}$,$\theta_{\rm ds})$)\label{line:findSimilarNode}
\If{$n_{\rm homo}$ = $\varnothing$}
\State $Q.{\rm push}(n_{\rm succ})$
\State $V(x_{\rm ds}, y_{\rm ds}, \theta_{\rm ds}).{\rm push}(n_{\rm succ})$
\Else
\If{$n_{\rm succ}.gCost < n_{\rm homo}.gCost$}
\If {$n_{\rm homo}$ is in $Q$}
\State remove $n_{\rm homo}$ from $Q$
\EndIf
\State $Q.{\rm push}(n_{\rm succ})$
\State $V(x_{\rm ds}, y_{\rm ds}, \theta_{\rm ds}).{\rm push}(n_{\rm succ})$
\EndIf
\EndIf
\EndFor
\EndWhile
\State $R = \varnothing$
\State \Return $R$
\end{algorithmic}  
\end{algorithm}

\section{Algorithm}
\label{section_algorithm}

In this section, the SEFTPP problem is effectively addressed using a constrained path searching algorithm. 
The pseudo code of the proposed algorithm is shown in \textbf{Algorithm~\ref{alg:main_planner}}. 

\subsection{Node Definition}
The proposed algorithm constructs a searching tree of valid tethered robot configurations that satisfy all the constraints stated in the previous subsection. 
The path searching process is similar to the constrained searching-based optimal path planner~\cite{Dolgov2010Path}~\cite{Bhattacharya2012Topological}. 
To formally present this, the term \textit{node} is defined as follows: 

\begin{definition}
(Node) A \textit{node} during the pathfinding consists of the following elements: 
\begin{equation}
\begin{aligned}
n = \{&\{i, x, y, \theta, O, s, \phi, h\} (\mbox{configuration related}), \\
&\qquad \{gCost, hCost, i_{\rm prev}\} (\mbox{searching-tree related}), \\
&\qquad\qquad \{steer, dir\} (\mbox{cost related})\} \\
\end{aligned}
\end{equation}
where $i$ is the index of the node, $\{x, y, \theta, O\}$ is the pose of the mobile unit. 
The location of tether-robot anchoring point $s$ and the tether stretching orientation $\phi$ are derived variables from the robot configuration. 
$h$ is the $h$-signature of the robot configuration, which is also a derived variable and whose calculation will be elaborated upon later. 
Other components are introduced specifically for node expansion, which are presented in the next subsection. 
\end{definition}

\subsection{Node Expansion}
At the beginning of the algorithm, a representative point for each obstacle is firstly distinguished, denoted as $\zeta_1, \cdots, \zeta_n$. 
Then, parallel non-overlapping rays are constructed, denoted as $r_1, \cdots, r_n$. 
The initial configuration of the robot (which must satisfy all constraint) is employed to construct the first node. 
This node, denoted as $n_0$, is constructed with cost-to-move $n_0.gCost = 0$ and $h$-signature $n_0.h = \varnothing$. 
It is then pushed into a priority queue. 
In each iteration, the node $n_{\rm cur}$ with the lowest cost is popped from the queue, and all of its child nodes are generated. 
Multiple primitive paths may be applicable. 
In the particular SEFTPP case, given that in-situ rotational movements rarely comply with the SEF constraint, the primitive paths are selected as car-like circular paths, parameterised by $dis$ (path length), $dir$ (where $1$ represents forward and $-1$ represents backward), and $steer$ (which determines the turning radius). 
The validity of a primitive movement is rigorously examined, with ``validity" encompassing not only the collision-free property, tether-length-admissibility, tether non-selfcrossing property, and self-entanglement-free property of the child node, but also extending to all intermediate waypoint configurations along the primitive path. 
The set of valid child nodes is denoted as $N_{\rm succ}$. 
The node expansion process is detailed in \textbf{Algorithm~\ref{alg:succ}}~\footnote{The module is the same as the Algorithm 2 in~\cite{Yang2023Self}, but is presented in a slightly lengthy form for the easy re-usage in the \textbf{Algorithm~\ref{alg:improved_node_expansion}} of this paper.
Also, the pseudocode is arranged next to \textbf{Algorithm~\ref{alg:improved_node_expansion}} for easy comparison. }. 

\begin{figure}[t]
\centering
\includegraphics[width=0.38\textwidth]{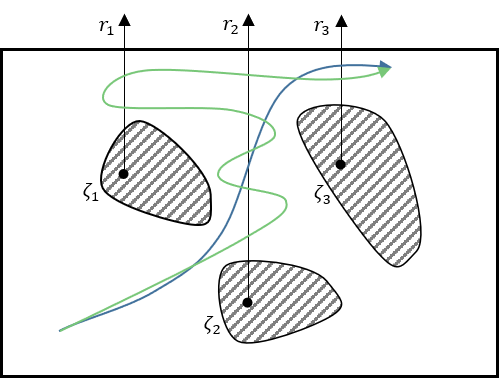}
\caption{Illustration of $h$-signature calculation. 
In this example, the word representation of the green path and the blue path are ``$r_2r_2^{-1}r_2r_2^{-1}r_1^{-1}r_1r_2r_3$" and ``$r_2r_3$", respectively, which are equivalent under the process of word reduction. 
Consequently, the two paths are homotopic. If the tethered robot moves along the two paths, the configurations will be exactly the same.  
}
\label{fig:hsign}
\end{figure}

For each valid child node, elements of the robot configuration is calculated, denoted as $n_{\rm succ}.x$, $n_{\rm succ}.y$, $n_{\rm succ}.\theta$, $n_{\rm succ}.O$. 
If the primitive path traverses across $r_i$ from left to right, then $r_i$ is appended to the $h$-signature, whilst $r_i^{-1}$ is appended if the crossing is from right to left. 
If $r_i$ and $r_i^{-1}$ are adjacent, then they are reduced. 
Through this process, $n_{\rm succ}.h$ is calculated based on $n_{\rm cur}.h$. 
See Fig.~\ref{fig:hsign} for illustration. 
The author is referred to~\cite{Kim2014Path} for a more theoretical explanation. 
Various non-negative movement cost functions may be applicable for calculating the cost. 
In our implementation, the cost is calculated as a composite function that considers the robot's travelling distance, changes in steering angle, and alternations in moving direction, as 
\begin{equation}
\begin{aligned}
n_{\rm succ}.gCost = &n_{\rm cur}.gCost + k_1 dis\\
&+ k_2\parallel n_{\rm succ}.steer - n_{\rm cur}.steer\parallel \\
&+ k_3\parallel n_{\rm succ}.dir - n_{\rm cur}.dir\parallel\\
&\qquad \qquad k_1, k_2, k_3 > 0
\end{aligned}
\end{equation}
where $k_1$, $k_2$, and $k_3$ are parameters. 
The index of the parent node is stored in the child node as $n_{\rm succ}.i_{\rm prev}$. 
The estimated cost-to-go, denoted as $n_{\rm succ}.hCost$, is calculated as the Euclidean distance to the goal. 
Nodes are managed in a resolutionally complete manner, meaning that the space of the mobile unit's poses $(x, y, \theta)$ is discretised based on pre-defined grid solutions $x_{\rm res}$, $y_{\rm res}$, $\theta_{\rm res}$. 
In cases where multiple nodes with the same $h$-signature reside in the same grid, only the one with the lowest cost is preserved. 
However, it is permitted to have multiple nodes with pairwise distinct $h$-signature within the same grid. 

The algorithm iteratively expands the least-cost node in the queue, until it reaches the goal location. 
A valid path is then reported by back-tracing through the nodes following the child-parent relation. 

\subsection{Discussions on Completeness and Distance-Optimality}

\textbf{Completeness. }
The completeness property of a path planning algorithm means that the algorithm is guaranteed to either find a resultant path or to determine that no path exists in finite time. 
The completeness of the proposed algorithm follows the similar vein as that of the Hybrid A*~\cite{Dolgov2010Path} algorithm. 
By discretising the map into small grids (regarding $x$, $y$, and $\theta$), allowing the existence of multiple nodes with distinct $h$-signature within the same node, and choosing short primitive paths, the algorithm can effectively find a resultant path within a reasonable computation time. 
However, there is no formal guarantee that the searching branch of the proposed algorithm can explore all grids within the same connected component of the grid containing the initial node. 

\textbf{Optimality. }
The proposed algorithm is sub-optimal. 
Evaluating the quality of the resultant path involves two perspectives: local (among all paths within the same homotopy class of paths) and global (among the best-possible paths found in each individual homotopy class). 
From a local standpoint, the optimality of the proposed algorithm is sub-optimal. 
This is because if multiple nodes in the same grid have homotopic tether shapes, the proposed algorithm only preserves the least-cost one, disregarding the others. 
This behaviour aligns with observations made in prior searching-based algorithms~\cite{Dolgov2010Path}. 
From the global perspective, the proposed algorithm maintains path searching branches in all homotopy class of paths, enabling it to identify and compare sub-optimal paths across multiple homotopy classes and select the best one. 

\section{Sparse Waypoint Validity Checking}\label{section:sparsity}

The proposed path searching algorithm adopts straight movement and arc-like movements as primitive paths. 
However, a critical concern is how the validity of the primitive paths can be efficiently validated. 
The most straightforward strategy is discretising the primitive path into a sufficiently dense sequence of waypoint configurations of the tethered robot, based on a set distance resolution, and verifying the validity of each individual waypoint configuration. 
Nonetheless, this process is extremely inefficient due to the computational cost associated with the explicit calculation of the tether shape for each waypoint configuration. 
Given that the validity checking module is executed during the expansion of every child node, the inefficiency of this module directly impacts the efficiency of the overall algorithm. 
In this section, it is proven that under specific conditions, there is no necessity to examine the validity of the waypoint configurations: they are guaranteed to be valid. 
This is referred to as the \textit{sparsity} of the validity checking. 

To simplify the discussions presented in this section, our scope is limited to the situations where the contact points between the tether and obstacles remain unchanged throughout the primitive motion. 
In this regard, we first show the existence of a method for identifying the constancy of tether-obstacle contact points. 
\begin{lemma}
(Unchanged Tether-Obstacle Contact Points) 
Let the starting configuration be $c_0 = (x_0, y_0, \theta_0, O_0)$ and the ending configuration be $c_1 = (x_1, y_1, \theta_1, O_1)$. 
If 
\begin{enumerate}
\item $O_0$ and $O_1$ are the identical. 
The last tether-obstacle contact point is referred to as $o$. 
\item Using $o$ and the path of $s$ to generate a convex hull, this convex hull is obstacle-free. 
\end{enumerate}
then the tether-obstacle contact points will remain unchanged throughout the primitive motion. 
\end{lemma}
\begin{proof}
Because the final part of the tether (between $o$ and $s$) always resides within this convex hull, it will not encounter any obstacle. 
\end{proof}

\subsection{The Sparsity of Validity Checking}

Typically, there are four properties that must be checked to determine whether a primitive path is valid: the collision-free property, the non-selfcrossing (NS) property, the self-entanglement-free (SEF) property, and the tether-length-admissible (TLA) property. 

To begin with, verifying the collision-free property only involves the verification that the footprint of the mobile unit does not intersect with any obstacle during the motion. 
It is essentially the collision checking module used for untethered differential-driven robots which has been an established module. 
Next, it is observed that the starting segment of the tether (from the base point $b$ to $o$) remains unchanged throughout the motion. 
Therefore, verifying the non-selfcrossing property is simply implemented as verifying the collision-free property between the robot path and the starting part of the tether (from $b$ to $o$). 
The non-selfcrossing starting configuration implies that the mobile unit will hit the static part of the tether before the tether becomes selfcrossing. 
See Fig.~\ref{fig:selfcrossing} for illustration. 

Checking the SEF property and the TLA property of a primitive path are challenging because these tasks are directly related to the deformed shape of the final part of the tether (from $o$ to $s$). 
As a result, the kernel of the sparsity of the validity checking is exploring conditions under which the properties of the primitive path can be fully characterised by the endpoint configurations. 
This is formally presented as follows. 

\begin{figure}[t]
\centering
\includegraphics[width=0.4\textwidth]{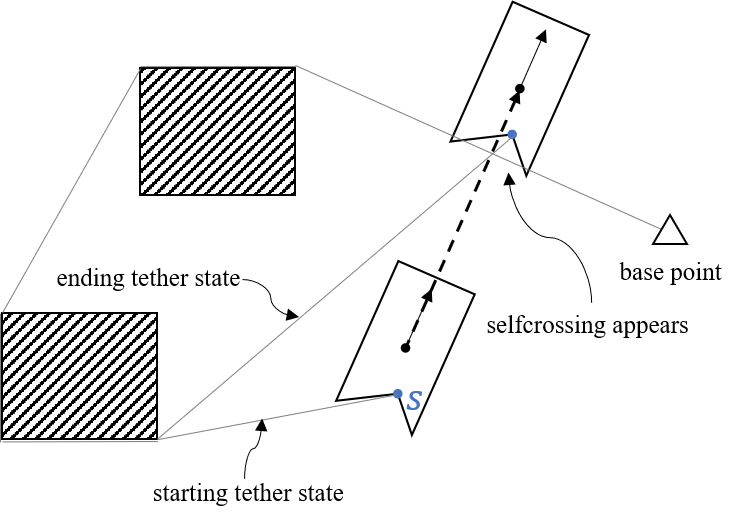}
\caption{Illustration of a selfcrossing robot path.}
\label{fig:selfcrossing}
\end{figure}

\begin{definition}\label{def:sparsity}
(Sparsity of Validity Checking) 
Let the robot path be a straight path or a circular path. 
The starting configuration is given as $c_0 = (x_0, y_0, \theta_0, O)$. 
The ending configuration is denoted as $c_1$. 
The sparsity of validity checking is defined as the sufficiency that
\begin{align}
c_0\mbox{ and }c_1\mbox{ are SEF}&\Rightarrow \mbox{ all waypoints are SEF}\label{eqn:sef_condition}\\
c_0\mbox{ and }c_1\mbox{ are TLA}&\Rightarrow \mbox{ all waypoints are TLA}\label{eqn:tla_condition}
\end{align}
\end{definition}

When the sparsity is established, there is no requirement to explicitly construct waypoint configurations, leading to significant computational time savings. 
The circumstances under which the sparsity is confirmed, in other words, the sufficient conditions for the sparsity, are discussed in the subsequent subsections. 
To enhance clarity, the following notations are formally recalled. 

\begin{definition}
($\Delta x$, $\Delta y$) The location of the tether-robot contact point $s$ in the robot's egocentric frame is denoted as ($\Delta x$, $\Delta y$). 
\end{definition}

\begin{definition}
($\tilde{x}$, $\tilde{y}$) The location of $s$ in the world frame is denoted as ($\tilde{x}$, $\tilde{y}$). 
\end{definition}

\begin{definition}
($o_x$, $o_y$) The location of the last tether-obstacle contact point $o$ in the world frame is denoted as ($o_x$, $o_y$). 
\end{definition}

\begin{figure}[t]
\centering
\subfigure[Illustration of a tethered straight primitive path. ]{
\includegraphics[width=0.48\textwidth]{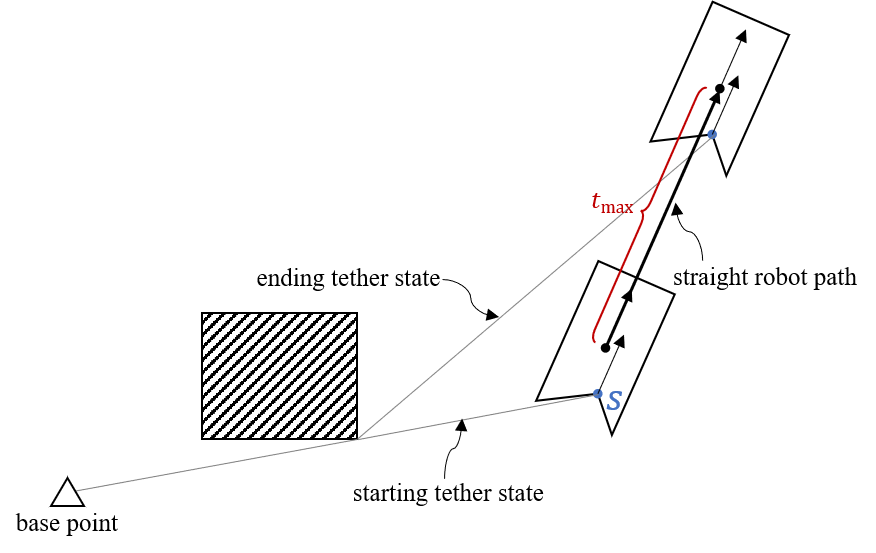}
\label{fig:primitive_motions:a}
}
\subfigure[Illustration of a tethered forward and left-turning primitive path. ]{
\includegraphics[width=0.48\textwidth]{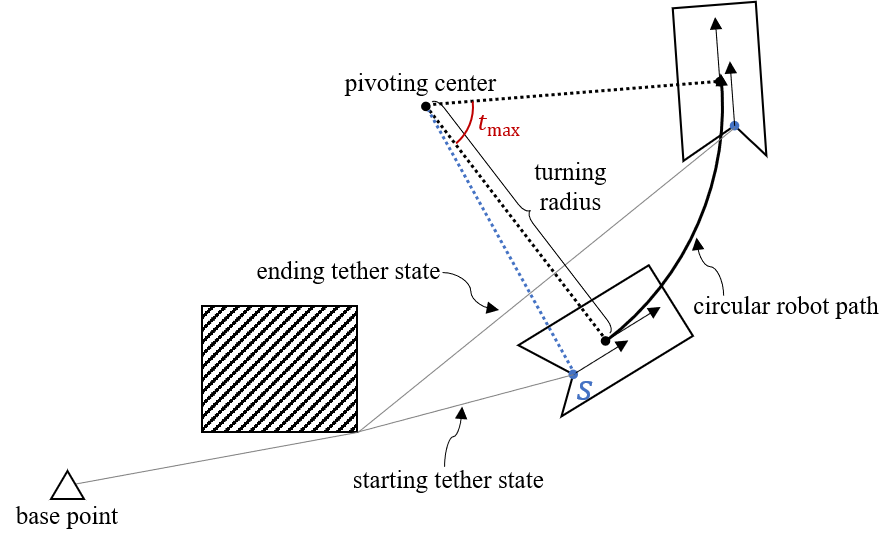}
\label{fig:primitive_motions:b}
}
\caption{Illustration of the starting configuration and the ending configuration of two primitive paths with an unchanged last tether-obstacle contact point. 
}\label{fig:primitive_motions}
\end{figure}

\subsection{The Monotonicity of Relative Angle Function}\label{sec:mono_SEF}

On noticing the significance of endpoint configurations in a primitive path, it is natural to consider the cases when the relative angle function and the tether length function are monotonic. 
Taking $\Phi$ as an example. 
If $\Phi$ is monotonic and the boundary values of $\Phi$ are within the admissible interval $[\Phi_1, \Phi_2]$, then all values of $\Phi$ also lie within the admissible interval. 
In other words, the monotonicity of the relative angle function is a sufficient condition that makes Eqn.~(\ref{eqn:sef_condition}) correct.  

As an introductory proposition, the monotonicity of the relative angle function when the robot path is a straight path is proven in detail. 

\begin{theorem}\label{thm:SEF_Straight}
(Relative Angle Monotonicity, Straight) 
Let the robot path be a Straight path. 
Then the relative angle function is monotonic. 
\end{theorem}
\begin{proof}
See Fig.~\ref{fig:primitive_motions} for illustration. 
The robot path is parameterised as 
\begin{align}
x(t) &= x_0 + t\cos\theta_0\\
y(t) &= y_0 + t\sin\theta_0\\
\theta(t) &= \theta_0
\end{align}
where $t$ is the arc-length parameter, and $t_{\rm max}$ is the length of the straight path. 
Then, the path of of $s$ is 
\begin{align}
&\tilde{x}(t) = x_0 + t\cos\theta_0 + \cos\theta_0\Delta x - \sin\theta_0\Delta y\\
&\tilde{y}(t) = y_0 + t\sin\theta_0 + \sin\theta_0\Delta x + \cos\theta_0\Delta y
\end{align}
The relative angle function is expressed as follows
\begin{equation}
\Phi(t) = \arctan\left( \frac{o_y - \tilde{y}(t)}{o_x-\tilde{x}(t)} \right) - \theta(t),\ t\in (0, t_{\rm max})
\end{equation}
Calculating the derivative of $\tilde{x}$, $\tilde{y}$, and $\theta$ with respect to $t$ gives
\begin{align}
\frac{d\tilde{x}}{dt} &= \cos\theta_0\\
\frac{d\tilde{y}}{dt} &= \sin\theta_0\\
\frac{d\theta}{dt} &= 0
\end{align}
The derivative of $\Phi$ is calculated as 
\begin{equation}
\frac{d\Phi}{dt} = \frac{-\frac{d\tilde{y}}{dt}(o_x-\tilde{x})+\frac{d\tilde{x}}{dt}(o_y-\tilde{y})}{(o_y - \tilde{y})^2 + (o_x - \tilde{x})^2} -\frac{d\theta}{dt}
\end{equation}
Before we advance further, it should be noted that we are only interested in the monotonicity of $\Phi$, specifically the comparison between $\frac{d\Phi}{dt}$ and $0$. 
The denominator $(o_x-\tilde{x})^2 + (o_y - \tilde{y})^2$ is always a positive value, rendering it irrelevant to our focus. 
Hence, the denominator is safely ignored by introducing $\frac{d\tilde{\Phi}}{dt}$ as follows:  
\begin{equation}\label{eqn:SEF_straight_tilde_Phi}
\begin{aligned}
\frac{d\tilde{\Phi}}{dt} \triangleq& \left( (o_x-\tilde{x})^2 + (o_y - \tilde{y})^2\right) \frac{d\Phi}{dt}\\
=& -\frac{d\tilde{y}}{dt}(o_x-\tilde{x}) + \frac{d\tilde{x}}{dt}(o_y-\tilde{y})\\
=& -\sin\theta_0(o_x - x_0 - t\cos\theta_0 - \cos\theta_0\Delta x+\sin\theta_0 \Delta y)\\
& +\cos\theta_0(o_y - y_0 - t\sin\theta_0 - \sin\theta_0\Delta x - \cos\theta_0 \Delta y)\\
=& -\sin\theta_0(o_x - x_0 - \cos\theta_0\Delta x+\sin\theta_0 \Delta y)\\
&+\cos\theta_0(o_y - y_0 - \sin\theta_0\Delta x - \cos\theta_0 \Delta y)
\end{aligned}
\end{equation}
Surprisingly, $\frac{d\tilde{\Phi}}{dt}$ is not a function of $t$, meaning that
\begin{equation}
\frac{d\tilde{\Phi}}{dt} > 0\mbox{ or }\frac{d\tilde{\Phi}}{dt} < 0\mbox{ or }\frac{d\tilde{\Phi}}{dt} = 0,\ \forall t\in (0, t_{\rm max})
\end{equation}
therefore 
\begin{equation}
\frac{d\Phi}{dt} > 0\mbox{ or }\frac{d\Phi}{dt} < 0\mbox{ or }\frac{d\Phi}{dt} = 0,\ \forall t\in (0, t_{\rm max})
\end{equation}
In all the cases, $\Phi$ is monotonic, implying that its maximum and minimum value are achieved at endpoints. 
\end{proof}

For circular primitive paths, the discussion is spanned based on four different path types: Forward and Right-turning, Forward and Left-turning, Backward and Right-turning, and Backward and Left-turning. 

\begin{theorem}\label{thm:SEF_FR}
(Relative Angle Monotonicity, F-R) Let the robot path be a Forward and Right-turning arc. 
The centre angle of the arc is $t_{\rm max}$. 
Then the relative angle function is monotonic if one of the following equations is satisfied, 
\begin{align}
&\frac{A^2+B^2}{\sqrt{C^2+D^2}} > \max\{\cos(t - \theta_0 - \varphi)\},\ \forall t\in (0, t_{\rm max})\label{eqn:FR_cond1}\\
&\frac{A^2+B^2}{\sqrt{C^2+D^2}} < \min\{\cos(t - \theta_0 - \varphi)\},\ \forall t\in (0, t_{\rm max})\label{eqn:FR_cond2}
\end{align}
where 
\begin{align}
&A = o_x - x_0 - R\sin\theta_0\\
&B = o_y - y_0 - R\cos\theta_0\\
&C = A\Delta x + BR + B\Delta y\\
&D = AR +A\Delta y - B\Delta x
\end{align}
and $\varphi$ is the angle such that 
\begin{equation}
\cos\varphi = \frac{C}{\sqrt{C^2+D^2}},\ \sin\varphi = \frac{D}{\sqrt{C^2+D^2}}
\end{equation}
\end{theorem}
\begin{proof}
See \textbf{Appendix~\ref{sec:appendix_SEF_FR}}. 
\end{proof}

\begin{theorem}\label{thm:SEF_FL}
(Relative Angle Monotonicity, F-L) 
Let the robot path be a Forward and Left-turning arc. 
The centre angle of the arc is $t_{\rm max}$. 
Then the relative angle function is monotonic if one of the following equations is satisfied,
\begin{align}
&\frac{A^2+B^2}{\sqrt{C^2+D^2}} > \max\{\cos(t+\theta_0-\varphi)\},\ \forall t\in (0, t_{\rm max})\\
&\frac{A^2+B^2}{\sqrt{C^2+D^2}} < \min\{\cos(t+\theta_0-\varphi)\},\ \forall t\in (0, t_{\rm max})
\end{align}
where
\begin{align}
&A = o_x - x_0 + R\sin\theta_0\\
&B = o_y - y_0 + R\cos\theta_0\\
&C = A\Delta x - BR + B\Delta y\\
&D = AR - A\Delta y +B\Delta x
\end{align}
and $\varphi$ is the angle such that 
\begin{equation}
\cos\varphi = \frac{C}{\sqrt{C^2+D^2}},\ \sin\varphi = \frac{D}{\sqrt{C^2 + D^2}}
\end{equation}
\end{theorem}
\begin{proof}
The proof is very similar to that of \textbf{Theorem~\ref{thm:SEF_FR}}, except many changes in the signs, which is put in \textbf{Appendix~\ref{sec:appendix_SEF_FL}}. 
\end{proof}

Given the fact that a B-L path is the inverse of a F-R path, and a B-R path is the inverse of a F-L path, the following remarks are concluded. 

\begin{remark}\label{rem:SEF_BL}
(Relative Angle Monotonicity, B-L) Let the robot path be a Backward and Left-turning arc. 
The centre angle of the arc is $t_{\rm max}$. 
Then the relative angle function is monotonic if the condition in \textbf{Theorem~\ref{thm:SEF_FR}} is satisfied, with the interval of $t$ being changed from $(0, t_{\rm max})$ to $(-t_{\rm max}, 0)$. 
\end{remark}

\begin{remark}\label{rem:SEF_BR}
(Relative Angle Monotonicity, B-R)
Let the robot path be a Backward and Right-turning arc. 
The centre angle of the arc is $t_{\rm max}$. 
Then the relative angle function is monotonic if the condition in \textbf{Theorem~\ref{thm:SEF_FL}} is satisfied, with the interval of $t$ being changed from $(0, t_{\rm max})$ to $(-t_{\rm max}, 0)$. 
\end{remark}

\subsection{The Monotonicity of Tether Length Variation}\label{sec:mono_TLA}

In this subsection, sufficient conditions that guarantee the monotonicity of the tether length function are presented. 

\begin{theorem}
(Tether Length Monotonicity, Straight) Let the robot path be a Straight path. 
The length of the path is $t_{\rm max}$. 
Then the tether length function is monotonic if one of the following equations is satisfied, 
\begin{align}
&o_x + o_y - x_0\cos\theta_0 - y_0\sin\theta_0 - \Delta x < 0\\
&o_x + o_y - x_0\cos\theta_0 - y_0\sin\theta_0 - \Delta x > t_{\rm max}
\end{align}
\end{theorem}
\begin{proof}
See \textbf{Appendix~\ref{sec:appendix_TLA_Straight}}. 
\end{proof}

\begin{algorithm}[t]
    \caption{Node Expansion}\label{alg:succ}
    \begin{algorithmic}[1]
        \Require current node $n_{\rm cur}$, Map $M$, Maximum tether length $L$, representative obstacle $\zeta$
        \Ensure Child node list $N_{\rm succ}$ 
\State $P$ = allPrimitives() \label{alg:node_expansion:prior_start}
\State $N_{\rm succ} = \varnothing$
\For{each element $(steer, dir, dis)$ in $P$}
\State $P_{\rm waypoint}$ = generatePath($n_{\rm cur}, (steer, dir, dis)$)
\State \% $P_{\rm waypoint}$ is the $x$-$y$-$\theta$ value of the waypoints
\State $n$ = size($P_{\rm waypoint}$) \% $n$ is large
\State is\_path\_valid = true
\For{$i$=$1, \cdots, n$}
\If{$\sim$isCollisionFree($P_{\rm waypoint}$, $i$) $||$
\Statex \qquad \qquad $\sim$isNS($P_{\rm waypoint}$, $i$, $n_{\rm cur}.O$)}
\State is\_path\_valid = false
\State \textbf{break}
\EndIf
\EndFor
\If{$\sim$is\_path\_valid}
\State \textbf{continue}
\EndIf\label{alg:node_expansion:prior_end}
\For{$i$=$1, \cdots, n$}\label{alg:node_expansion:valid_start}
\State $n_{\rm mid}$ = generateConf($n_{\rm cur}.O, P_{\rm waypoint}, i$)\label{alg:node_expansion:construction}
\If{$\sim$isSEF($n_{\rm mid}$) $||$ $\sim$isTLA($n_{\rm mid}$)}
\State is\_path\_valid = false
\State \textbf{break}
\EndIf
\EndFor
\If{$\sim$is\_path\_valid}
\State \textbf{continue}
\EndIf\label{alg:node_expansion:valid_end}
\State $n_{\rm succ}$ = generateConf($n_{\rm cur}.O, P_{\rm waypoint}, n$)\label{alg:node_expansion:post_start}
\State $n_{\rm succ}.gCost$ = movementCost($n_{\rm cur}$)
\State $n_{\rm succ}.h$ = calculateHsignature$(n_{\rm succ}, \zeta)$
\State $N_{\rm succ}$.push\_back($n_{\rm nucc}$)
\EndFor\label{alg:node_expansion:post_end}
\end{algorithmic}  
\end{algorithm}

\begin{algorithm}[t]
    \caption{Improved Node Expansion}\label{alg:improved_node_expansion}
    \begin{algorithmic}[1]
        \Require current node $n_{\rm cur}$, Map $M$, Maximum tether length $L$, representative obstacle $\zeta$
        \Ensure Child node list $N_{\rm succ}$ 
\State (line \ref{alg:node_expansion:prior_start} $\sim$ line \ref{alg:node_expansion:prior_end} in \textbf{Algorithm~\ref{alg:succ}})
\State $n_{\rm succ}$ = generateConf$(n_{\rm cur}.O, P_{\rm waypoint}, n)$
\If{$\sim$isSEF($n_{\rm succ}$) $||$ $\sim$isTLA($n_{\rm succ}$)}
\State \textbf{continue}
\EndIf
\State is\_o\_not\_changed = $\sim$isOChanged($n_{\rm cur}$, $M$, $P_{\rm waypoint}$)
\State is\_sef\_guaranteed = is\_o\_not\_changed \&\&
\Statex \qquad\qquad\qquad\qquad isSEFGuaranteed($steer$, $dis$)
\State is\_tla\_guaranteed = is\_o\_not\_changed \&\&
\Statex \qquad\qquad\qquad\qquad isTLAGuaranteed($steer$, $dis$)
\If{$\sim$is\_sef\_guaranteed $||$ $\sim$is\_tla\_guaranteed}
\State (line \ref{alg:node_expansion:valid_start} $\sim$ line \ref{alg:node_expansion:valid_end} in \textbf{Algorithm~\ref{alg:succ}})
\EndIf
\State (line \ref{alg:node_expansion:post_start} $\sim$ line \ref{alg:node_expansion:post_end} in \textbf{Algorithm~\ref{alg:succ}})
\end{algorithmic}  
\end{algorithm}

\begin{theorem}\label{thm:TLA_FR}
(Tether Length Monotonicity, F-R) 
Let the robot path be a Forward and Right-turning arc. 
The centre angle of the arc is $t_{\rm max}$. 
Then the tether length function is monotonic if one of the following equations is satisfied, 
\begin{align}
&\cos(t-\theta_0-\varphi) < 0,\ \forall t\in (0, t_{\rm max})\\
&\cos(t-\theta_0-\varphi) > 0,\ \forall t\in (0, t_{\rm max})
\end{align}
where $\varphi$ is the angle such that 
\begin{equation}
\cos\varphi = \frac{C}{\sqrt{C^2+D^2}},\ \sin\varphi = \frac{D}{\sqrt{C^2+D^2}}
\end{equation}
where
\begin{align}
C =& B\Delta x - AR - A\Delta y\\
D =& A\Delta x - BR - B\Delta y
\end{align}
and 
\begin{align}
A =& o_x - x_0 - R\sin\theta_0\\
B =& o_y - y_0 + R\cos\theta_0
\end{align}
\end{theorem}

\begin{proof}
See \textbf{Appendix~\ref{sec:appendix_TLA_FR}}. 
\end{proof}

\begin{theorem}\label{thm:TLA_FL}
(Tether Length Monotonicity, F-L) Let the robot path be a Forward and Left-turning arc. 
The centre angle of the arc is $t_{\rm max}$. 
Then the tether length function is monotonic if one of the following equations is satisfied, 
\begin{align}
&\cos(t+\theta_0 + \varphi) < 0,\ \forall t\in (0, t_{\rm max})\\
&\cos(t+\theta_0 + \varphi) > 0,\ \forall t\in (0, t_{\rm max})
\end{align}
where $\varphi$ is the angle such that 
\begin{equation}
\cos\varphi = \frac{C}{\sqrt{C^2+D^2}},\ \sin\varphi = \frac{D}{\sqrt{C^2+D^2}}
\end{equation}
where
\begin{align}
C =& A\Delta y - AR - B\Delta x\\
D =& A\Delta x - BR + B\Delta y
\end{align}
and 
\begin{align}
A =& o_x - x_0 + R\sin\theta_0\\
B =& o_y - y_0 + R\cos\theta_0
\end{align}
\end{theorem}
\begin{proof}
See \textbf{Appendix~\ref{sec:appendix_TLA_FL}}
\end{proof}

\begin{remark}\label{rem:TLA_BL}
(Tether Length Monotonicity, B-L) 
Let the robot path be a Backward and Left-turning arc. 
The centre angle of the arc is $t_{\rm max}$. 
Then the tether length function is monotonic if the condition in \textbf{Theorem~\ref{thm:TLA_FR}} is satisfied, with the interval of $t$ being changed from $(0, t_{\rm max})$ to $(-t_{\rm max}, 0)$. 
\end{remark}

\begin{remark}\label{rem:TLA_BR}
(Tether Length Monotonicity, B-R) 
Let the robot path be a Backward and Right-turning arc. 
The centre angle of the arc is $t_{\rm max}$. 
Then the tether length function is monotonic if the condition in \textbf{Theorem~\ref{thm:TLA_FL}} is satisfied, with the interval of $t$ being changed from $(0, t_{\rm max})$ to $(-t_{\rm max}, 0)$. 
\end{remark}

\begin{figure*}[t]
\centering
\subfigure[Case $1$]{
\includegraphics[width=0.18\textwidth]{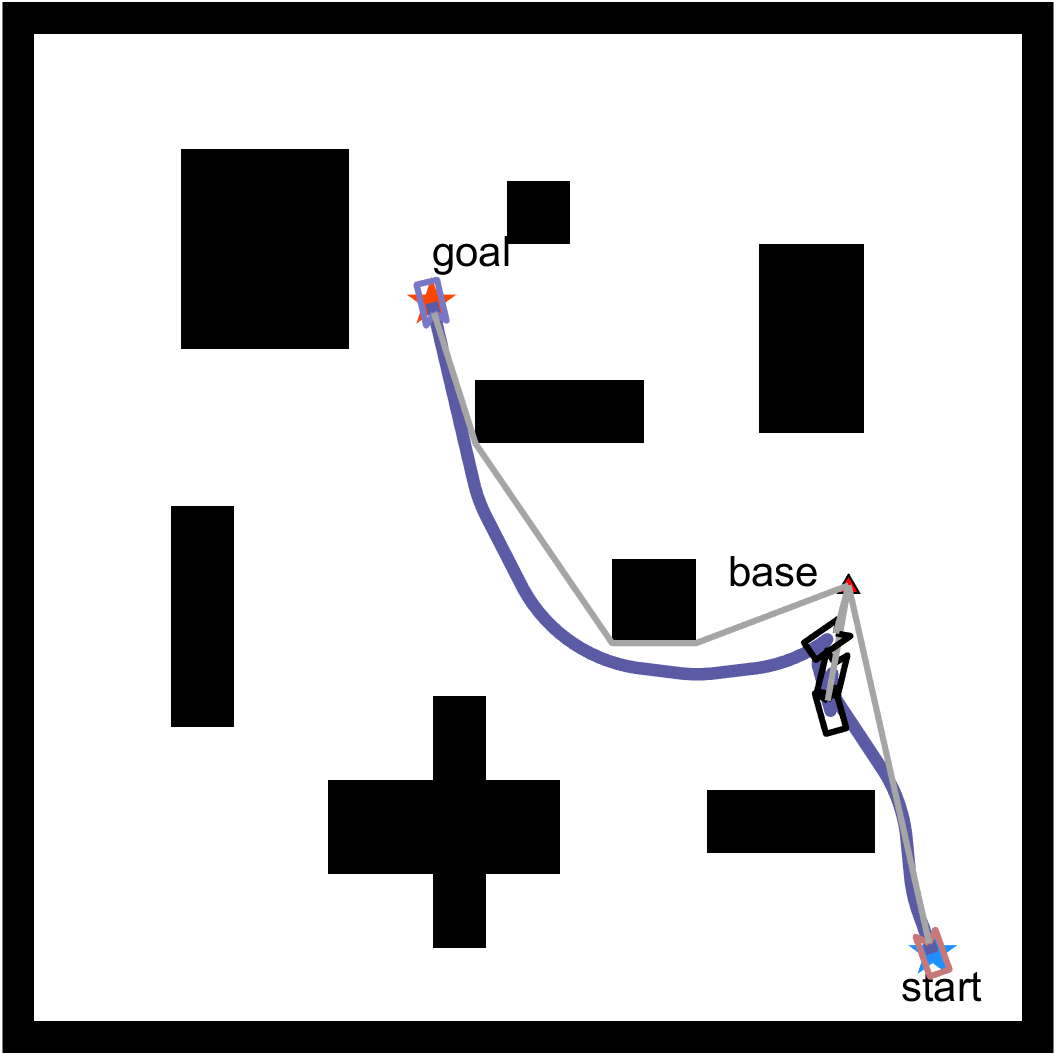}\label{fig:case1}
}
\subfigure[Case $2$]{
\includegraphics[width=0.18\textwidth]{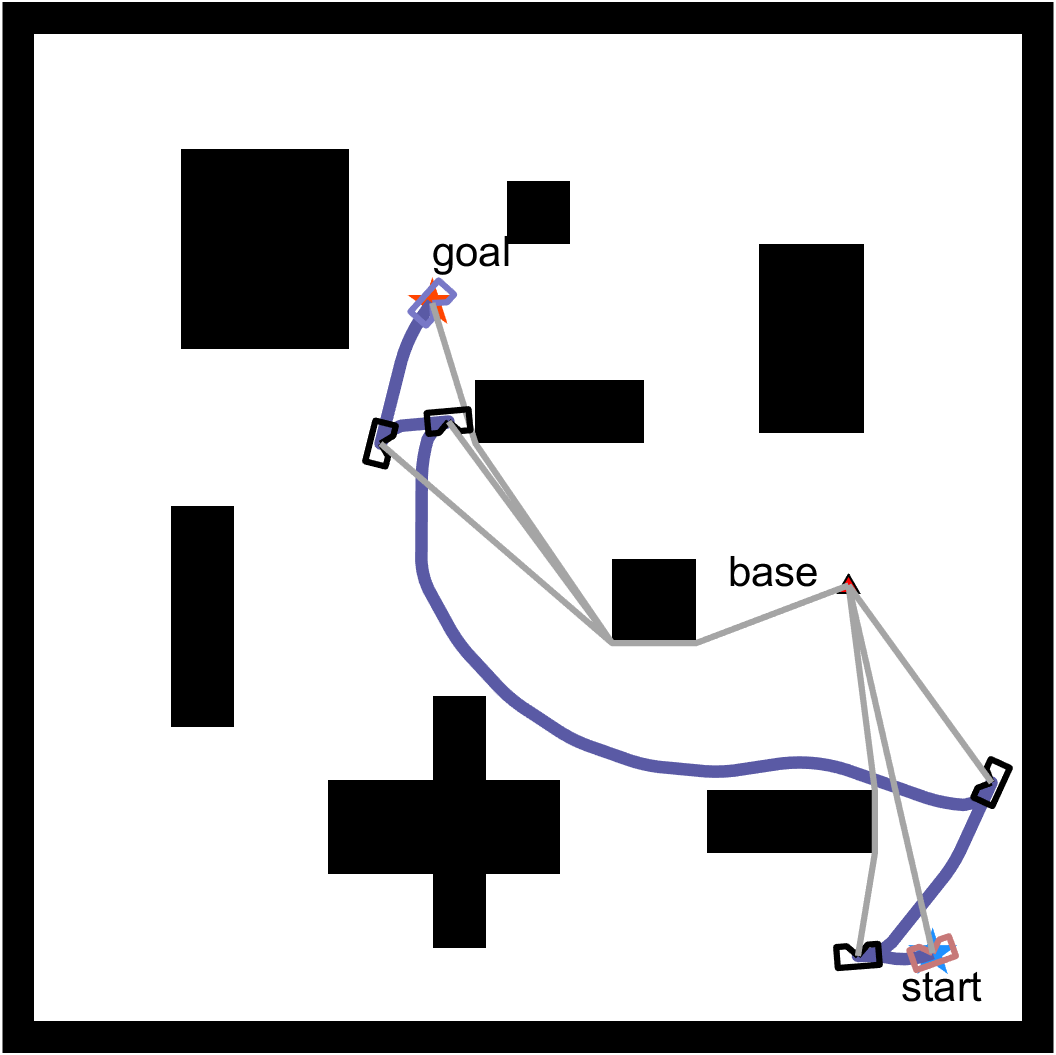}\label{fig:case2}
}
\subfigure[Case $3$]{
\includegraphics[width=0.18\textwidth]{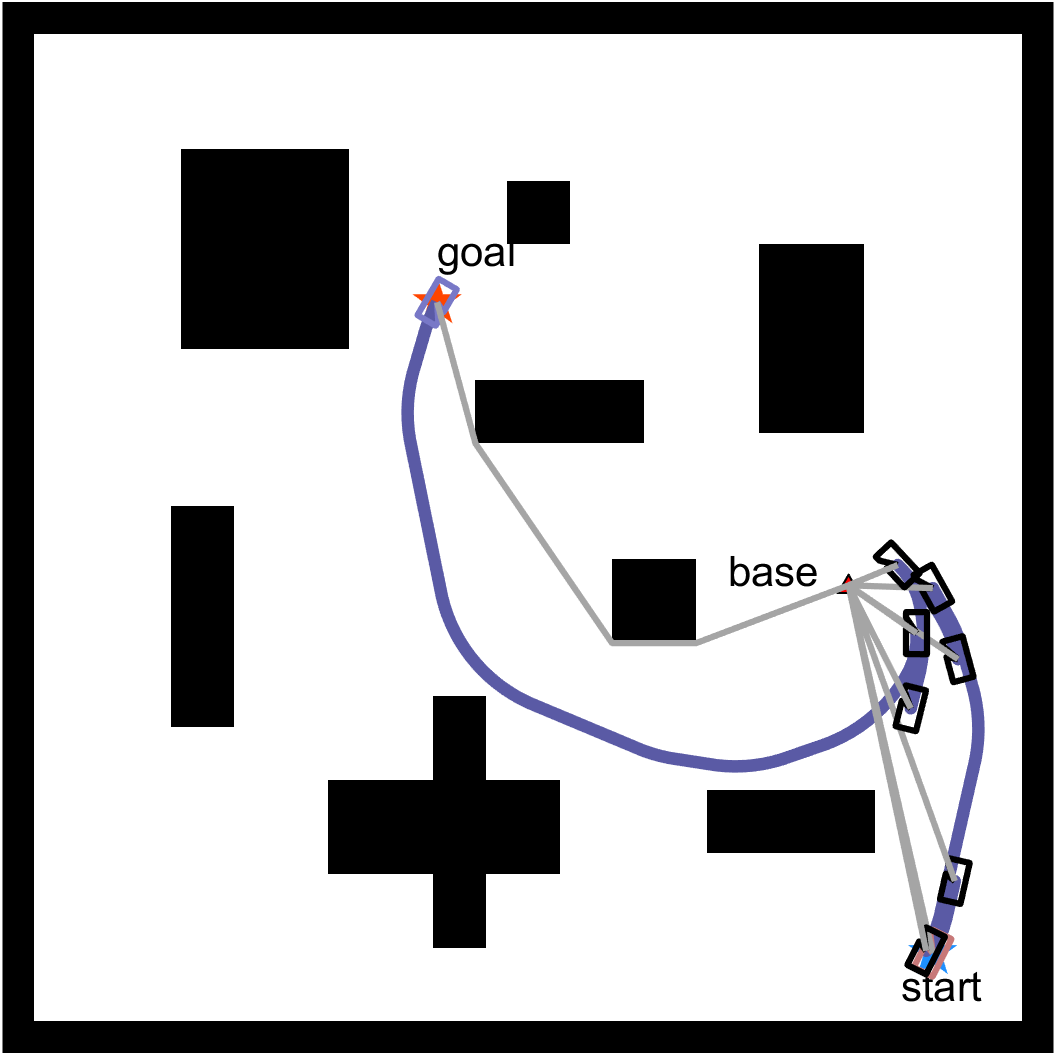}\label{fig:case3}
}
\subfigure[Case $4$]{
\includegraphics[width=0.18\textwidth]{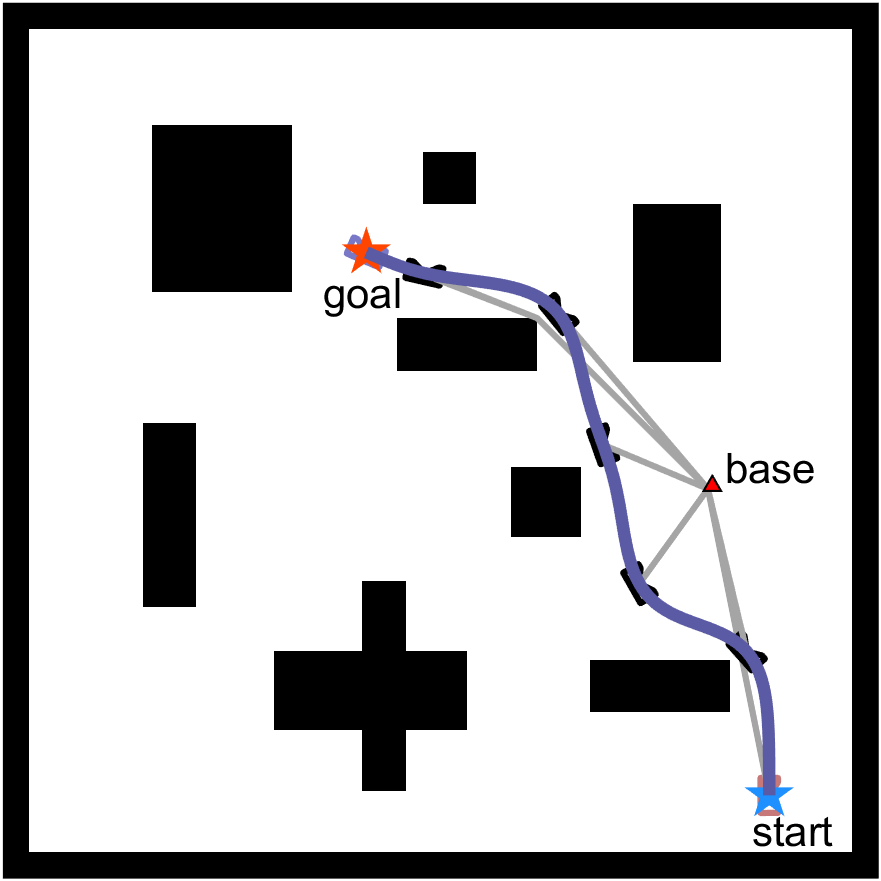}\label{fig:case1_4}
}
\subfigure[Case $5$]{
\includegraphics[width=0.18\textwidth]{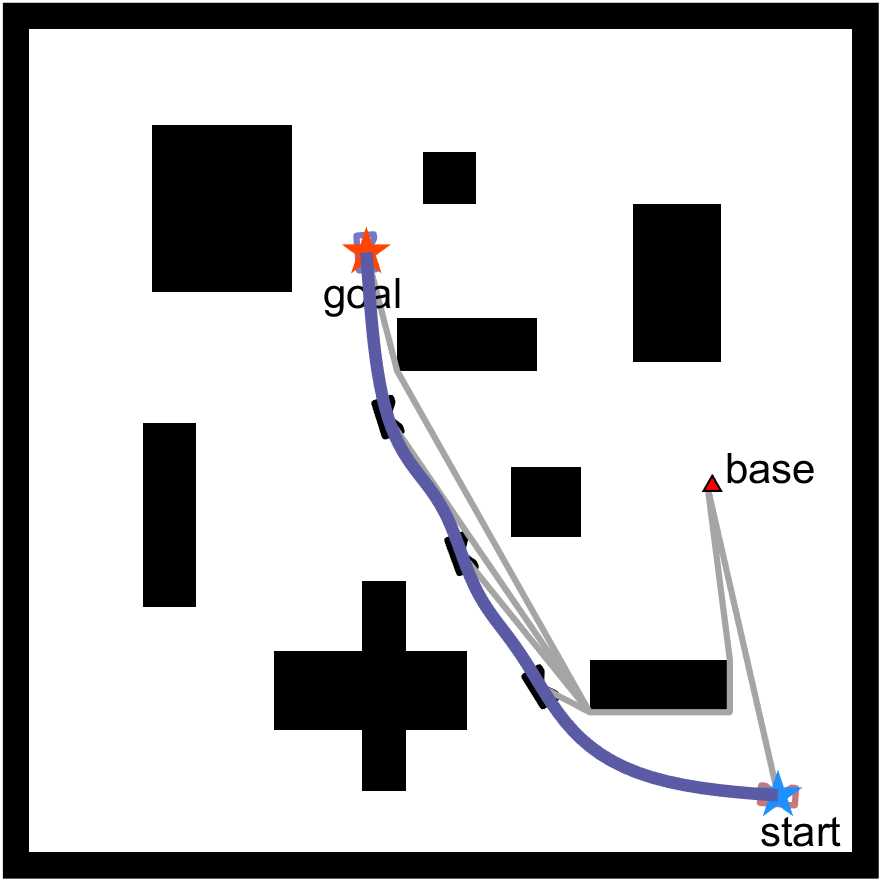}\label{fig:case1_5}
}
\caption{Illustration of tethered motions from the same start configuration to the same goal location. 
(a)(b)(c) Motions satisfying specified SEF constraints. 
(d)(e) Commonplace (untethered) differential-driven robot motions, where (d) violates the SEF constraint and (e) violates both the SEF constraint and the TLA constraint. 
}\label{fig:exp_case_studies}
\end{figure*}

\begin{figure}[t]
\centering
\subfigure[Variation of $\Phi$ in Case 1 (blue), 2 (red), and 3 (green). ]{
\includegraphics[width=0.48\textwidth]{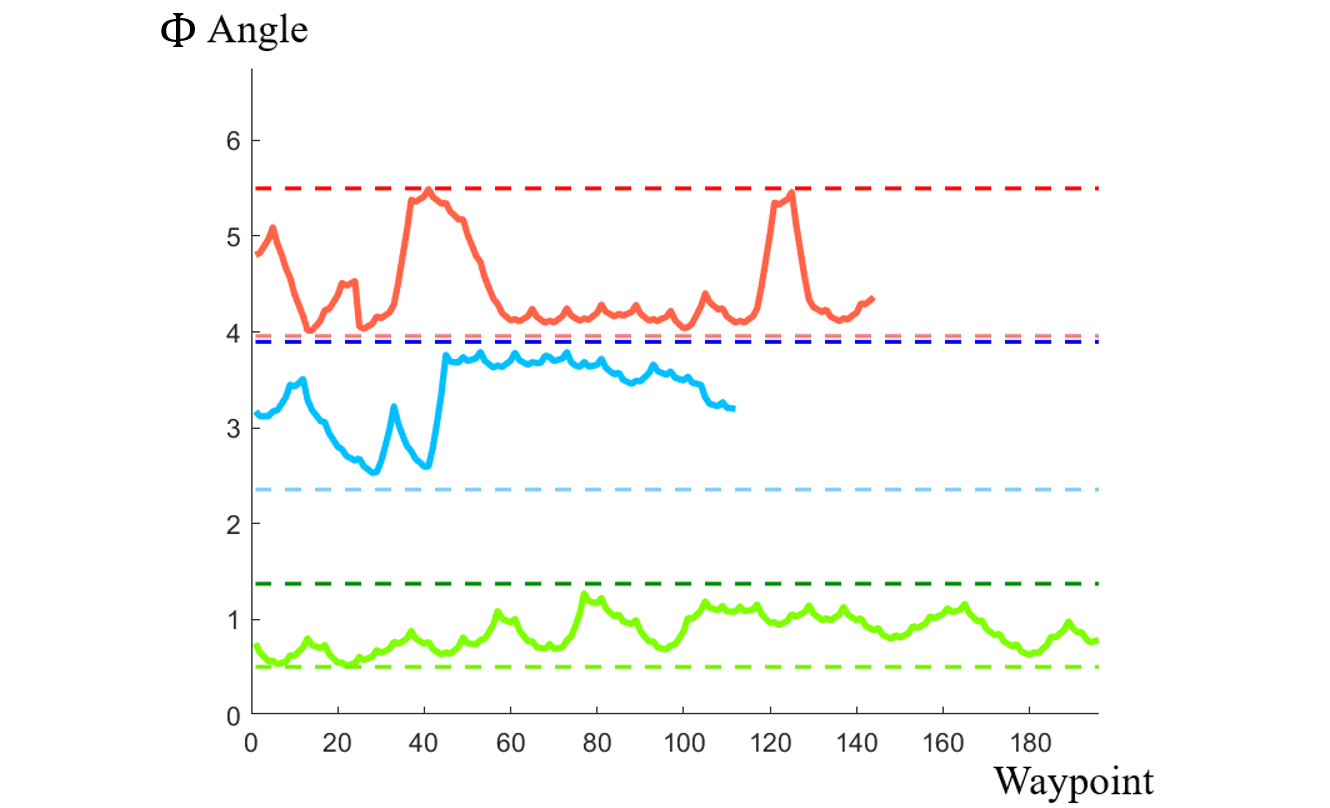}\label{fig:phi_case1_123}
}
\subfigure[Variation of $\Phi$ in Case 2 (red), 4 (dark grey), and 5 (light grey). ]{
\includegraphics[width=0.48\textwidth]{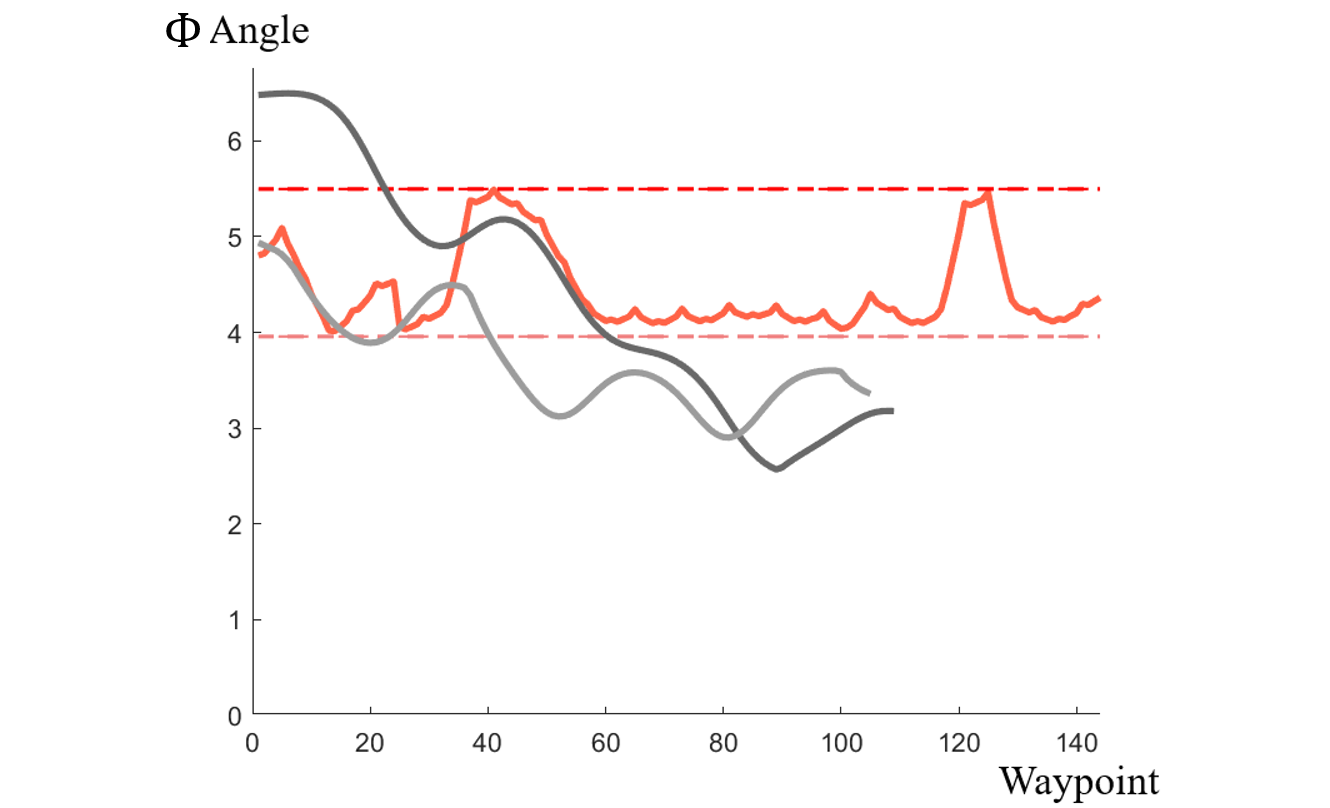}\label{fig:phi_case1_145}
}
\caption{Illustration of the angular difference between robot's heading orientation and the tether retracting direction. Admissible upper bounds and lower bounds are depicted. 
}\label{fig:exp_Phi}
\end{figure}

\subsection{Summary (Improved Node Expansion)}
As the result of the aforementioned discussions, the established sufficient conditions are incorporated into \textbf{Algorithm~\ref{alg:succ}}. 
The improved node expansion module is presented in \textbf{Algorithm~\ref{alg:improved_node_expansion}}. 
The most apparent difference between the two algorithms is the re-location of the most time-consuming command, the construction of all waypoint configurations (line~\ref{alg:node_expansion:construction} in \textbf{Algorithm~\ref{alg:succ}}), into an \textbf{if} structure. 
After the primitive path is generated, the collision-free property and the non-selfcrossing property of the path are validated. 
Whether the tether-obstacle contact points are changing is also assessed. 
Then, based on the type of the robot path (straight, F-R, F-L, B-R, or B-L), known parameters are $\Delta x$, $\Delta y$, $o_x$, $o_y$, $x_0$, $y_0$, $\theta_0$, $R$, and $t_{\rm max}$. 
By calculating the corresponding $A$, $B$, $C$, $D$, and $\varphi$, the corresponding conditions (to be precise, one among \textbf{Theorem~\ref{thm:SEF_FR}}, \textbf{Theorem~\ref{thm:SEF_FL}}, \textbf{Remark~\ref{rem:SEF_BL}}, and \textbf{Remark~\ref{rem:SEF_BR}} for SEF, and one among \textbf{Theorem~\ref{thm:TLA_FR}}, \textbf{Theorem~\ref{thm:TLA_FL}}, 
\textbf{Remark~\ref{rem:TLA_BL}}, \textbf{Remark~\ref{rem:TLA_BR}} for TLA) are verified. 
These are reflected in the Boolean variables ``is\_o\_not\_changed", ``is\_sef\_guaranteed", and ``is\_tla\_guaranteed" as either ``True" or ``False". 
If all these variables evaluate to ``True", then all waypoint configurations are guaranteed to be valid, avoiding the need for explicit construction and inspection. 
The computational cost is hereby reduced. 

\section{Experimental Results}
\label{section_exp}

The proposed algorithm is designed to construct self-entanglement-free paths for tethered differential-driven robots with polygonal mobile unit operating in arbitrary planar environment. 
To the best of the authors' knowledge, there did not exist a prior self-entanglement-free tethered path planner. 
Therefore, in Section~\ref{sec:case_study}, the resultant paths generated by the proposed algorithm are demonstrated, alongside those produced by a commonly used differential-driven planner which does not consider the SEF constraint. 
In Section~\ref{section_exp_comparison}, the efficiency of the improved node expansion module is closely evaluated by comparing the computational time of both node expansion strategies. 
This assessment is conducted under various settings, including different robot kinematics, different lengths of primitive paths, and different distance resolutions for waypoint configuration validity checking. 
Finally, four real-world demonstrations are provided in Section~\ref{section_realworld} to validate the practicality of the proposed algorithm in real-world settings. 
An open-sourcing implementation has been provided here:

\url{https://github.com/ZJUTongYang/seftpp}. 

\subsection{Case Studies}\label{sec:case_study}
Refer to Fig.~\ref{fig:exp_case_studies} for illustrations. 
These illustrations take place on a $100\times 100$ grid-based planar map containing 8 internal obstacles. 
The base point, the start location, and the goal location are set at $(80.50, 44.50)$, $(88.50, 9.50)$, and $(41.50, 71.50)$, respectively. 
These locations are unchanged throughout all testing scenarios. 
The initial robot tether state is initialised as the local shortening of the Dijkstra's shortest path from the base point to the start location. 
The initial robot heading direction is initialised such that the tether stretching direction aligns with the middle of the admissible interval. 
The maximum tether length constraint is set at $80$ (grids).
Base point is illustrated as a triangle and is treated as an obstacle. 
Tether states are depicted as grey lines, whilst robot paths are drawn as thick blue curves.
In Case 1, 2, and 3, various tethered robot kinematics are demonstrated, with the tether extending to the back, right, and left, respectively. 
The corresponding admissible intervals $[\Phi_1, \Phi_2]$ are set as follows: [2.36, 3.93], [3.93, 5.50], [0.51, 1.11], respectively. 
Utilising the proposed algorithm, the resultant paths are guaranteed to adhere to the SEF constraint, as depicted in Fig.~\ref{fig:exp_case_studies}(a)$\sim$(c). 
Notably, Case 3 is the most difficult task for the robot to execute because the admissible $\Phi$ interval $[0.51, 1.11]$ is the narrowest. 
The variation of $\Phi$ during the robot motion is visualised in Fig.~\ref{fig:exp_Phi}(a). 
In contrast, using a path planner without explicit consideration of the SEF constraint, the resultant paths, as visualised in Fig.~\ref{fig:exp_case_studies}(d)$\sim$(e), violate the SEF constraint. 
The corresponding $\Phi$ variations are shown in Fig.~\ref{fig:exp_Phi}(b). 
The reader is referred to the supplementary video for the animation of the robot motions. 

\begin{table*}[t]
\setlength{\tabcolsep}{5pt}
\centering
\caption{Computational Time of Algorithms Given Different Primitive Path Length and Validity Checking Resolution in Case 1}\label{tab:data_in_cases}
\begin{tabular}{c|c|c|c|c|c}
\hline
\multicolumn{2}{c|}{} & \multicolumn{4}{c}{Distance Resolution between Consecutive Waypoints (grid)}\\
\cline{3-6}
\multicolumn{2}{c|}{}& 1.0 & 0.7 & 0.4 & 0.1\\
\hline
\multicolumn{2}{c|}{} & \multicolumn{4}{c}{Case 1}\\
\hline
\multirow{6}{*}{\rotatebox[origin=c]{90}{Primitive Path Length (grid)}}&1 & \makecell{ 7.40s / \textbf{6.60s} \\ 36529 / 715330 (\textbf{95.14\%})} & \makecell{ 4.40s / \textbf{4.19s} \\ 20778 / 451815 (\textbf{95.60\%})} & \makecell{ 8.17s / \textbf{6.76s} \\ 30080 / 576261 (\textbf{95.04\%})} & \makecell{ 33.10s / \textbf{14.76s} \\ 28760 / 480289 (\textbf{94.35\%})} \\
\cline{2-6}
& 2 &\makecell{ 8.98s / \textbf{7.04s} \\ 55438 / 508563 (\textbf{90.17\%}) }&\makecell{ 5.51s / \textbf{4.66s} \\ 32217 / 340386 (\textbf{91.35\%}) }&\makecell{ 12.55s / \textbf{8.21s} \\ 38614 / 369212 (\textbf{90.53\%}) }&\makecell{ 66.94s / \textbf{24.12s} \\ 34864 / 327846  (\textbf{90.39\%}) }\\
\cline{2-6}
& 3 &\makecell{ 7.89s / \textbf{6.59s} \\ 52063 / 328729 (\textbf{86.33\%}) }&\makecell{ 8.23s / \textbf{6.37s} \\ 40220 / 270946 (\textbf{87.07\%}) }&\makecell{ 13.39s / \textbf{8.56s} \\ 37649 / 251797 (\textbf{86.99\%}) }&\makecell{ 99.36s / \textbf{33.45s}\\ 36424 / 236765 (\textbf{86.67\%}) }\\
\cline{2-6}
& 4 &\makecell{8.93s / \textbf{6.66s} \\ 47769 / 236709 (\textbf{83.21\%}) }&\makecell{ 8.87s / \textbf{6.52s} \\ 42076 / 208312 (\textbf{83.20\%}) }&\makecell{ 19.36s / \textbf{11.74s} \\ 43679 / 212518 (\textbf{82.95\%}) }&\makecell{ 139.65s / \textbf{47.70s} \\ 39562 / 189695 (\textbf{82.74\%}) }\\
\cline{2-6}
& 5 &\makecell{ 8.48s / \textbf{6.42s} \\ 43403 / 170840 (\textbf{79.74\%}) }&\makecell{ 11.22s / \textbf{7.90s} \\ 41722 / 163427 (\textbf{79.66\%}) }&\makecell{ 20.57s / \textbf{12.15s} \\ 40558 / 156883 (\textbf{79.46\%}) }&\makecell{ 180.59s / \textbf{65.81s}\\ 41084 / 154362 (\textbf{78.98\%}) }\\
\cline{2-6}
& 6 &\makecell{ 6.18s / \textbf{4.78s} \\ 26571 / 92546 (\textbf{77.69\%}) }&\makecell{ 7.64s / \textbf{5.53s} \\ 24308 / 84090 (\textbf{77.58\%}) }&\makecell{ 16.54s / \textbf{10.03s}\\ 25054 / 84908 (\textbf{77.22\%}) }&\makecell{ 138.15s / \textbf{52.19s} \\ 24049 / 80592 (\textbf{77.02\%}) }\\
\hline
\multicolumn{2}{c|}{} & \multicolumn{4}{c}{Case 2}\\
\hline
\multirow{6}{*}{\rotatebox[origin=c]{90}{Primitive Path Length (grid)}}& 1 &\makecell{ \textbf{2.62s} / 2.63s\\ 23789 / 251896 (\textbf{91.37\%}) }&\makecell{ 2.05s / \textbf{2.04s} \\  20203 / 191871 (\textbf{90.47\%}) }&\makecell{ 2.56s / \textbf{2.30s} \\ 21594 / 160804 (\textbf{88.16\%}) }&\makecell{ 10.50s / \textbf{5.57s} \\ 21293 / 128949 (\textbf{85.83\%}) }\\
\cline{2-6}
& 2 &\makecell{ 3.26s / \textbf{2.91s} \\ 43406 / 167300 (\textbf{79.40\%}) }&\makecell{ 1.60s / \textbf{1.50s} \\ 22051 / 77320 (\textbf{77.81\%}) }&\makecell{ 5.85s / \textbf{4.26s} \\ 41686 / 126885 (\textbf{75.27\%}) }&\makecell{ 19.93s / \textbf{9.70s} \\ 26907 / 74202 (\textbf{73.39\%}) }\\
\cline{2-6}
& 3 &\makecell{ 3.68s / \textbf{3.11s} \\ 52545 / 111075 (\textbf{67.89\%}) }&\makecell{ 3.56s / \textbf{3.03s} \\ 42840 / 83278 (\textbf{66.03\%}) }&\makecell{ 3.83s / \textbf{2.78s} \\ 27483 / 49670 (\textbf{64.38\%}) }&\makecell{ 23.02s / \textbf{12.18s} \\ 22432 / 38035 (\textbf{62.90\%}) }\\
\cline{2-6}
& 4 &\makecell{ 3.38s / \textbf{3.01s} \\ 46312 / 61572 (\textbf{57.07\%}) }&\makecell{ 1.82s / \textbf{1.69s}\\ 21813 / 26968 (\textbf{55.28\%}) }&\makecell{ 5.23s / \textbf{4.00s} \\ 30238 / 34761 (\textbf{53.48\%}) }&\makecell{ 23.67s / \textbf{14.03s} \\ 17949 / 19317 (\textbf{51.84\%}) }\\
\cline{2-6}
& 5 &\makecell{ 3.09s / \textbf{2.76s} \\ 37874 / 36126 (\textbf{48.82\%}) }&\makecell{ 3.07s / \textbf{2.61s} \\ 26745 / 25309 (\textbf{48.62\%}) }&\makecell{ 4.02s / \textbf{3.33s} \\ 18408 / 17048 (\textbf{48.08\%}) }&\makecell{ 21.72s / \textbf{12.88s} \\ 10968 / 11615 (\textbf{51.43\%}) }\\
\cline{2-6}
& 6 &\makecell{ 2.44s / \textbf{2.15s} \\ 24416 / 17965 (\textbf{42.39\%}) }&\makecell{ 1.58s / \textbf{1.36s} \\ 11490 / 8756 (\textbf{43.25\%}) }&\makecell{ 3.69s / \textbf{3.03s} \\ 13418 / 9905 (\textbf{42.47\%}) }&\makecell{ 18.17s / \textbf{12.49s} \\ 8162 / 5059 (\textbf{38.26\%}) }\\
\hline
\multicolumn{2}{c|}{} & \multicolumn{4}{c}{Case 3}\\
\hline
\multirow{6}{*}{\rotatebox[origin=c]{90}{Primitive Path Length (grid)}}& 1 &\makecell{ 1.86s / \textbf{1.80s} \\ 11368 / 169447 (\textbf{93.71\%}) }&\makecell{ 1.79s / \textbf{1.69s} \\ 9567 / 161868 (\textbf{94.42\%}) }&\makecell{ 2.47s / \textbf{2.22s} \\ 11054 / 155369 (\textbf{93.36\%}) }&\makecell{ 12.14s / \textbf{6.12s} \\ 14008 / 156054 (\textbf{91.76\%}) }\\
\cline{2-6}
& 2 &\makecell{ 1.99s / \textbf{1.74s} \\ 13448 / 90343 (\textbf{87.04\%}) }&\makecell{ 1.78s/ \textbf{1.64s} \\ 11187 / 78749 (\textbf{87.56\%}) }&\makecell{ 3.77s / \textbf{2.81s} \\ 13625 / 83647 (\textbf{85.99\%}) }&\makecell{ 19.29s / \textbf{8.63s} \\ 13112 / 76473 (\textbf{85.36\%}) }\\
\cline{2-6}
& 3 &\makecell{ 2.40s / \textbf{1.93s} \\ 12466 / 56802 (\textbf{82.00\%}) }&\makecell{ 2.76s / \textbf{1.78s} \\ 8541 / 43027 (\textbf{83.44\%}) }&\makecell{ 3.98s / \textbf{2.96s} \\ 10078 / 46166 (\textbf{82.08\%}) }&\makecell{ 25.58s / \textbf{12.22s} \\ 11487 / 47525 (\textbf{80.53\%}) }\\
\cline{2-6}
& 4 &\makecell{ 2.28s / \textbf{1.80s} \\ 9221 / 33008 (\textbf{78.16\%}) }&\makecell{ 2.18s / \textbf{1.87s} \\ 7746 29611 (\textbf{79.26\%}) }&\makecell{ 4.52s / \textbf{3.48s} \\ 9348 / 31045 (\textbf{76.86\%}) }&\makecell{ 30.14s / \textbf{14.92s} \\ 9120 / 30223 (\textbf{76.82\%}) }\\
\cline{2-6}
& 5 &\makecell{ 2.01s / \textbf{1.75s} \\ 7254 / 21316 (\textbf{74.61\%}) }&\makecell{ 2.46s / \textbf{2.13s} \\ 6576 / 20188 (\textbf{75.43\%}) }&\makecell{ 4.09s / \textbf{3.15s} \\ 5901 / 18973 (\textbf{76.28\%}) }&\makecell{ 27.58s / \textbf{13.60s} \\ 5773 / 18126 (\textbf{75.84\%}) }\\
\cline{2-6}
& 6 &\makecell{ 1.75s / \textbf{1.47s} \\ 5142 / 12850 (\textbf{71.42\%}) }&\makecell{ \textbf{1.78s} / \textbf{1.78s} \\ 4880 / 12341 (\textbf{71.66\%}) }&\makecell{ 3.45s / \textbf{3.18s} \\ 5121 / 12336 (\textbf{70.67\%}) }&\makecell{ 23.12s / \textbf{11.94s} \\ 3878 / 10477 (\textbf{72.99\%}) }\\
\hline
\end{tabular}
\begin{tablenotes}
\item All results have been averaged over $20$ runs for a fair evaluation. 
\end{tablenotes}
\end{table*}

\subsection{Efficiency of the Improved Node Expansion}\label{section_exp_comparison}
In line with the algorithm presented earlier, the computational advantage of the improved node expansion module is eliminating the need for validity checking on the waypoint configurations of primitive paths. 
Importantly, this computational improvement is not a trade-off: It makes no difference on the result of the node expansion process, including the number of nodes expanded, the order of expansion, parent-child relationship among nodes, and the optimality of the resultant path. 
To substantiate the computational improvement, the normal node expansion module and the improved node expansion module are compared. 
The comparative assessments consider different robot kinematics, different primitive path lengths, and different distance resolutions for validity checking. 

In this comparative study, the robot kinematics are set identical to the Case $1$(tether extending to the back), Case $2$(tether extending to the right), and Case $3$(tether extending to the left) in the previous case study. 
The lengths of primitive paths are varied across multiple tests, ranging from $1$ to $6$ grids. 
Each primitive path is discretely sampled into waypoint poses of the mobile unit, based on different distance resolutions $1.0$, $0.7$, $0.4$, and $0.1$. 
Relative statistics are collected in Table.~\ref{tab:data_in_cases} as follows: 
\begin{enumerate}
\item The top row reports the computational time of the normal node expansion module and the computational time of the improved node expansion module. 
\item The bottom row provides the number of primitive paths that require waypoint discretisation and detailed validity checking, the number of primitive paths whose validity is guaranteed, and the proportion of the guaranteed valid primitive paths among all primitive paths. 
\end{enumerate}

Examining each row in Table,~\ref{tab:data_in_cases} it can be seen that, when the distance resolution for validity checking is reduced, the time required for checking the validity of each primitive path increases. 
As a result, the computational time of both node expansion module are increasing. 
However, the computational time of the improved node expansion module experiences a slower rate of growth. 
Furthermore, when comparing data in the same column, another observation is that as the length of primitive paths increases, the likelihood of the guaranteed valid primitive paths decreases. 
In the best case, with a primitive path length of 1 grid, over 95\% of primitive paths are guaranteed to be valid. 
Notably, even with unusually lengthy primitive paths (6 grids, equivalent to 6\% of the size of the map), the improved node expansion module can safely ignore the validity checking of approximately 40\% of primitive paths. 
Finally, it is crucial to highlight that the improved node expansion module consistently reduces the computational load across all testing cases. 
This is because the replacement of the validity checking of waypoint configurations is simply a sequence calculation of float-point numbers, $A$, $B$, $C$, $D$, and $\varphi$. 
Result show that except the left-top test in Case 2, the improved node expansion offers computational advantages in all other testings.  

\begin{figure}[t]
\centering
\subfigure[robot structure]{
\includegraphics[width = 0.18\textwidth]{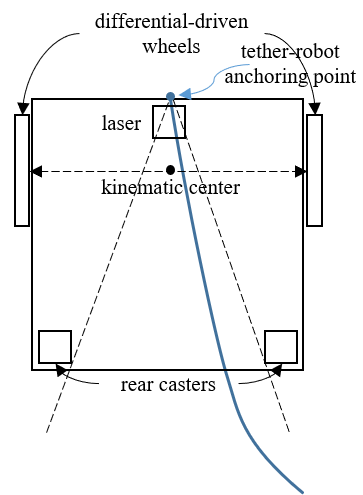}\label{fig:robot_kinematic}
}
\subfigure[the environment]{
\includegraphics[width = 0.18\textwidth]{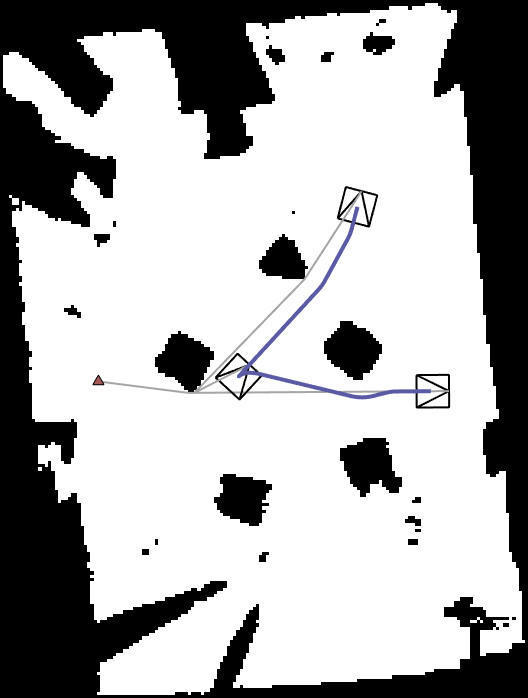}\label{fig:real_world_map}
}
\caption{(a) The real-world robot kinematics. 
The tether is anchoring at the front bottom of the robot chassis, below the laser, and extends to the back, in the middle of two rear casters. 
(b) The environment used for real-world tests, modelled by a gridmap. A resultant path is also depicted. }\label{fig:real_world}
\end{figure}

\begin{figure*}[t]
\centering
\subfigure[Initial Configuration]{
\includegraphics[width=0.23\textwidth]{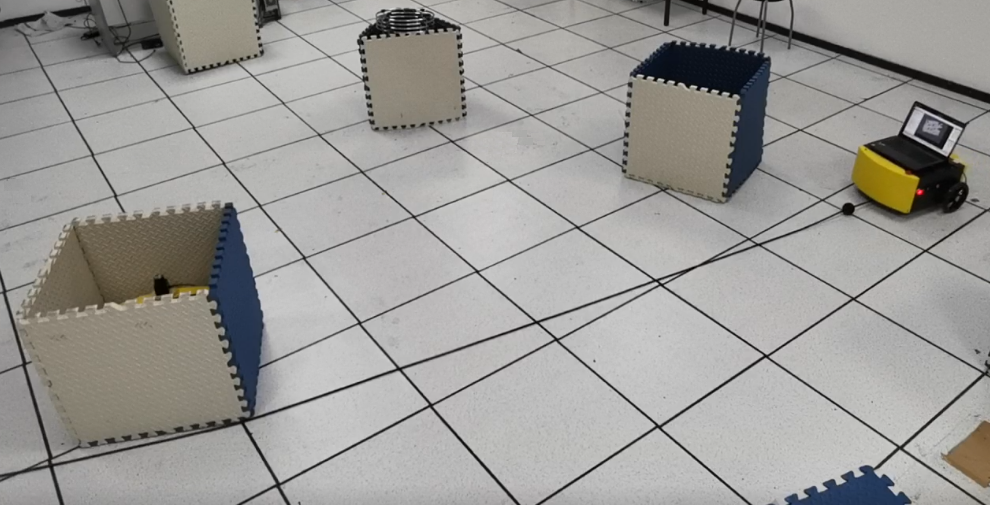}
}
\subfigure[Intermediate Configuration 1]{
\includegraphics[width=0.23\textwidth]{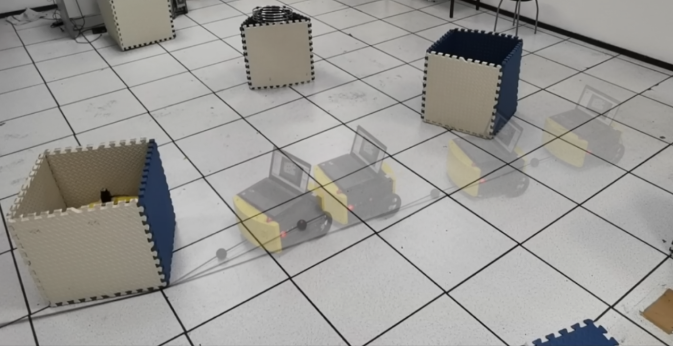}}
\subfigure[Intermediate Configuration 2]{
\includegraphics[width=0.23\textwidth]{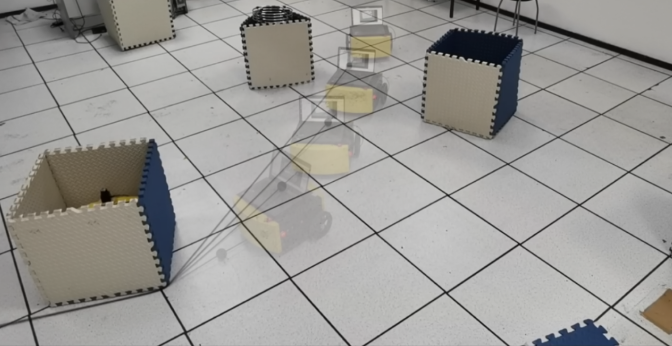}}
\subfigure[Target Configuration]{
\includegraphics[width=0.23\textwidth]{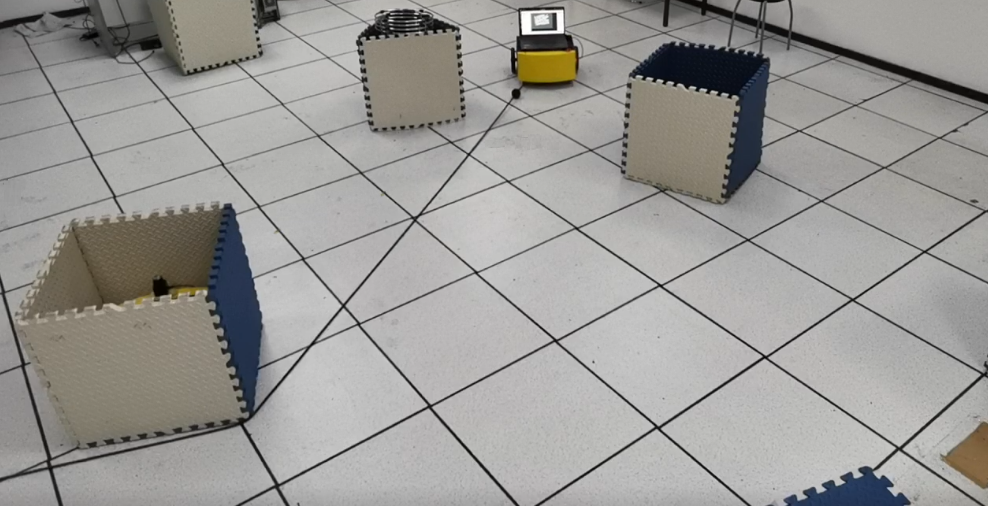}}

\subfigure[Initial Configuration ]{
\includegraphics[width=0.23\textwidth]{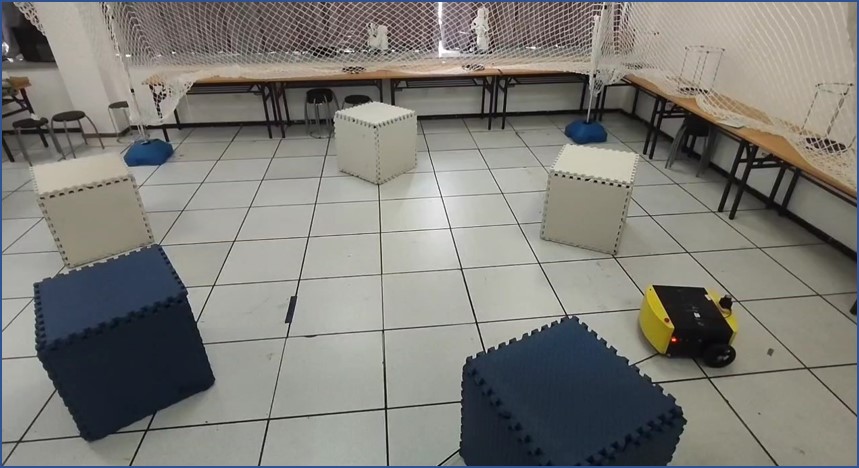}
}
\subfigure[Intermediate Configuration 1]{
\includegraphics[width=0.23\textwidth]{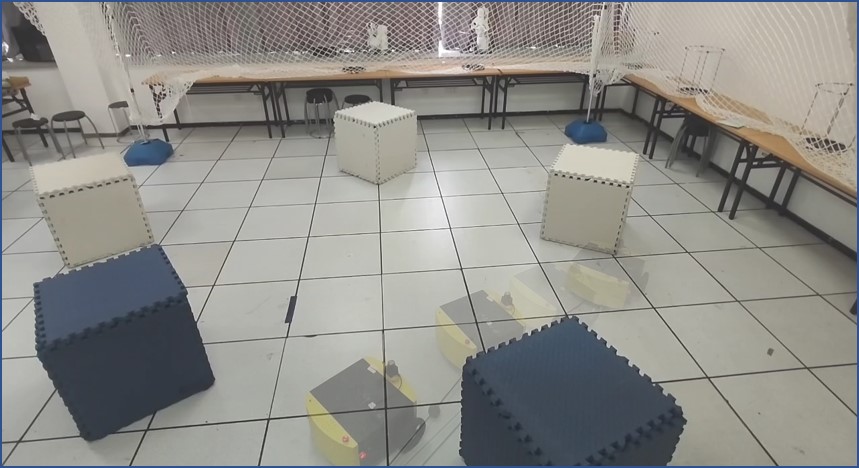}}
\subfigure[Intermediate Configuration 2]{
\includegraphics[width=0.23\textwidth]{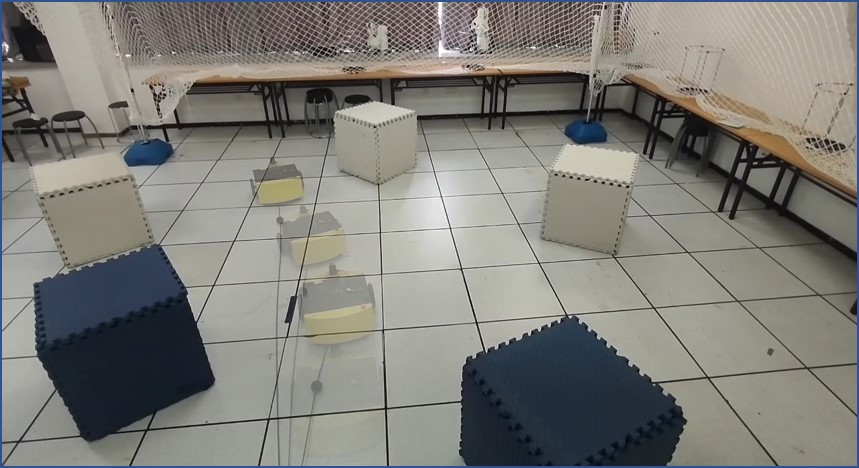}}
\subfigure[Target Configuration]{
\includegraphics[width=0.23\textwidth]{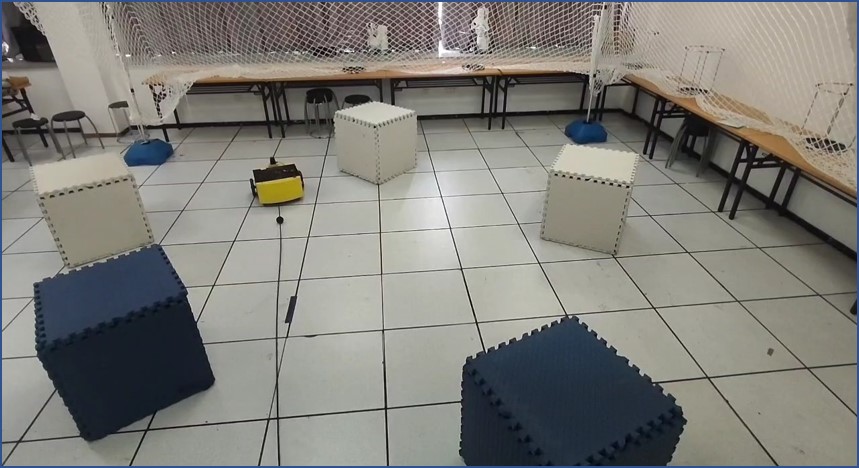}}

\subfigure[Initial Configuration]{
\includegraphics[width=0.23\textwidth]{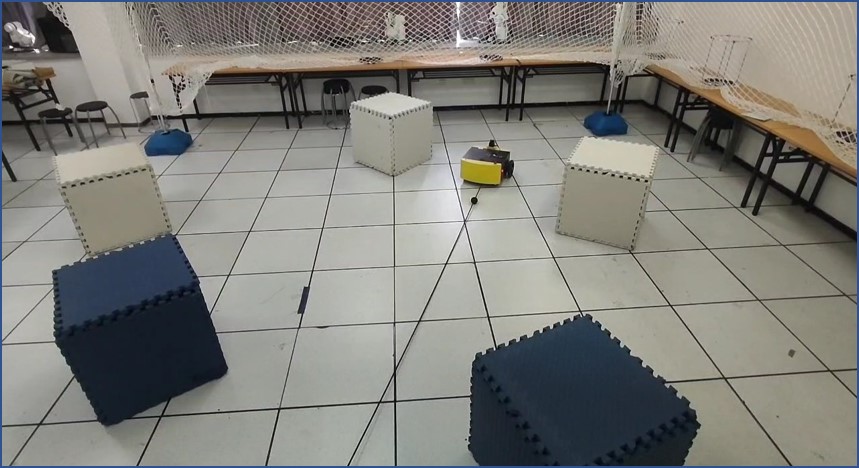}}
\subfigure[Intermediate Configuration 1]{
\includegraphics[width=0.23\textwidth]{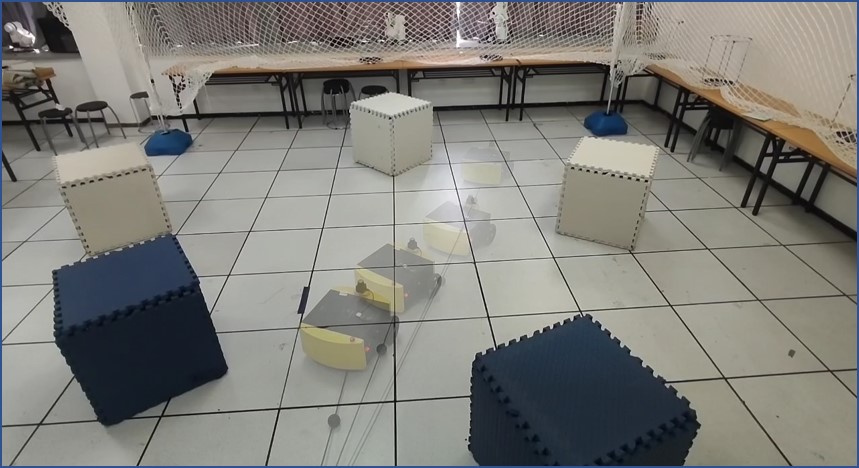}}
\subfigure[Intermediate Configuration 2]{
\includegraphics[width=0.23\textwidth]{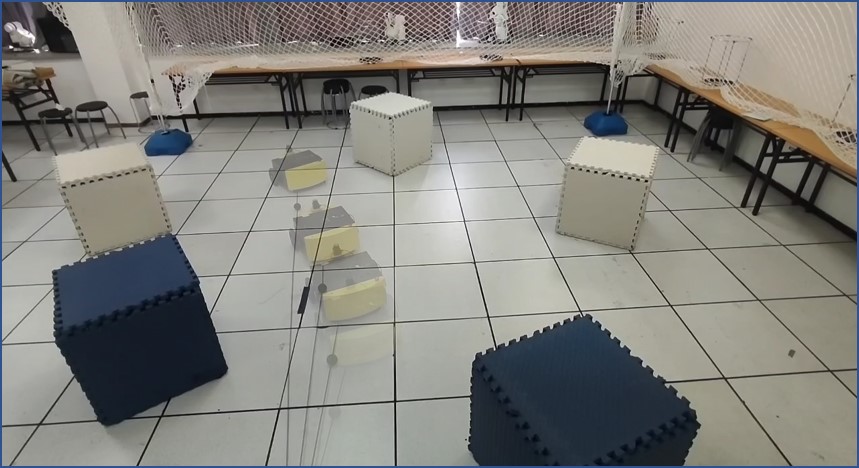}}
\subfigure[Target Configuration]{
\includegraphics[width=0.23\textwidth]{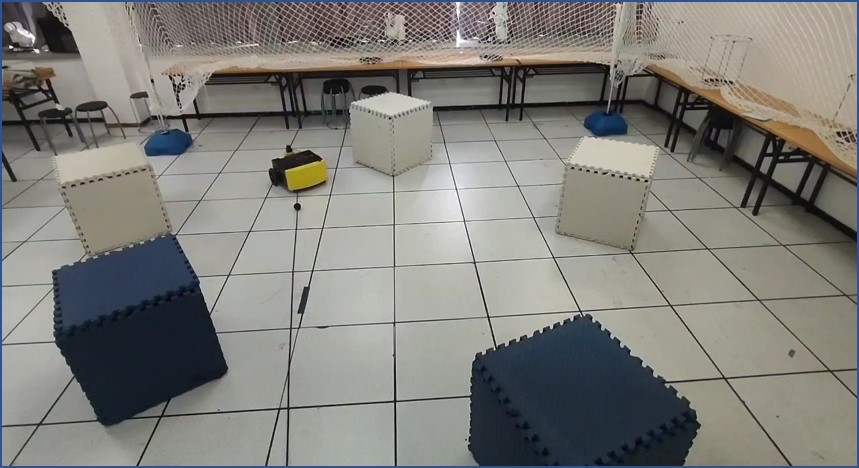}}

\subfigure[Initial Configuration]{
\includegraphics[width=0.23\textwidth]{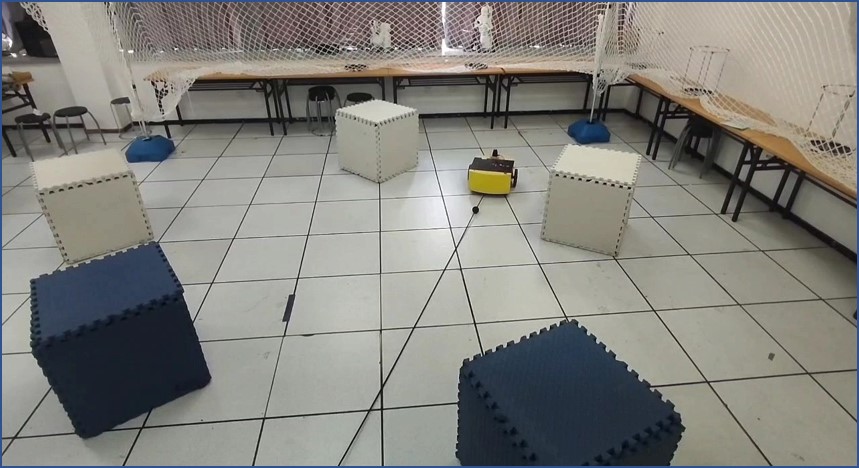}}
\subfigure[Intermediate Configuration 1]{
\includegraphics[width=0.23\textwidth]{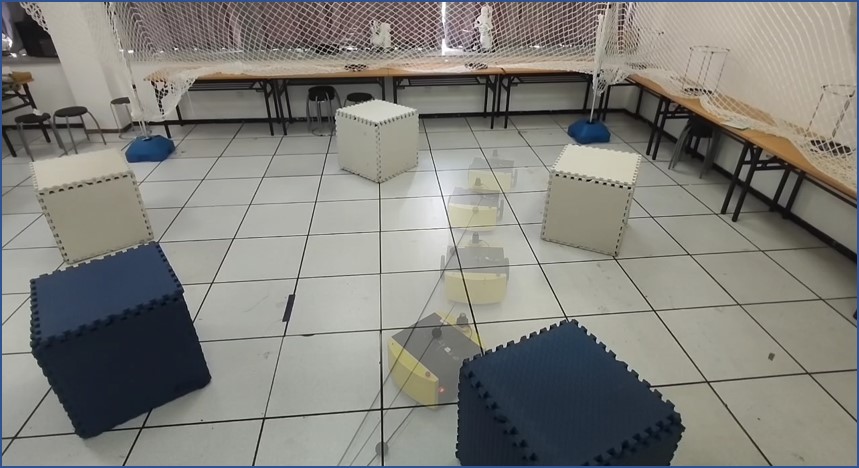}}
\subfigure[Intermediate Configuration 2]{
\includegraphics[width=0.23\textwidth]{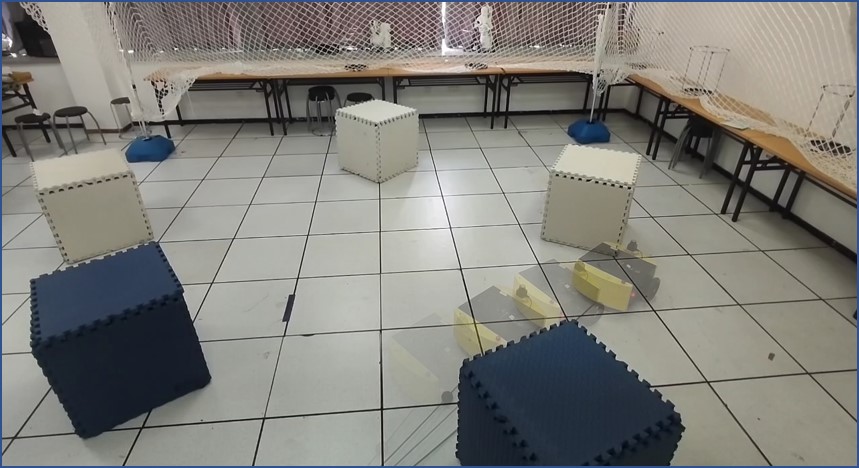}}
\subfigure[Target Configuration]{
\includegraphics[width=0.23\textwidth]{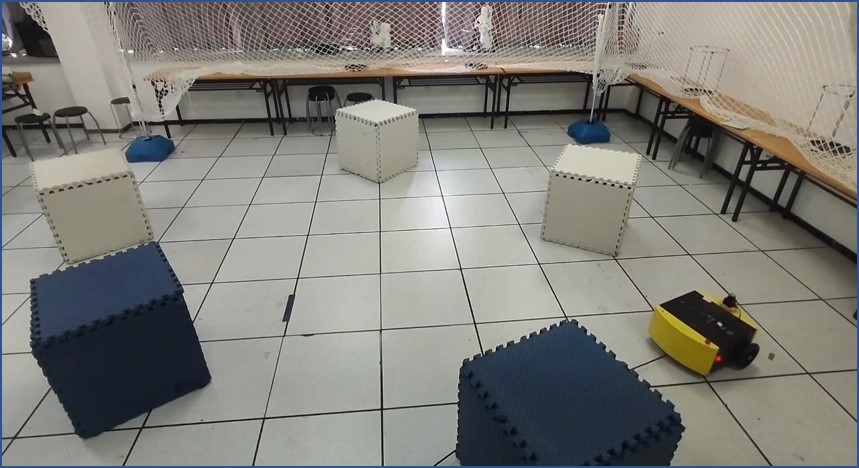}}
\caption{Video stills of the tethered robot motions to track the paths. }\label{fig:real_world_stills}
\end{figure*}

\subsection{Real-world Illustrations}\label{section_realworld}
Finally, the proposed algorithm is evaluated in four real-world scenarios. 
The structure of the robot is depicted in Fig.~\ref{fig:robot_kinematic}. 
The front wheels are differential-driven and the rear wheels are passive casters. 
No omni-directional tether retracting mechanism is equipped on the robot. 
Maps are off-line pre-constructed, as shown in Fig.~\ref{fig:real_world_map}.
Throughout all testing cases, the proposed algorithm successfully generates self-entanglement-free resultant paths for the differential-driven robot. 
The reader is referred to Fig.~\ref{fig:real_world_stills} for the illustrations of the real-world testings and the supplementary video for real-world executions. 

\section{Conclusion and Future Work}
\label{section_conclusion}
This work presents a novel mechanism for generating self-entanglement-free (SEF) paths for tethered differential-driven robots. 
The primary motivation of this work is the tethered path planning problem in the absence of an omni-directional tether-robot anchoring mechanism. 
No existing algorithm has previously addressed this self-entanglement phenomenon. 
The proposed algorithm is a searching-based constrained path planner that generates sub-optimal valid path for differential-driven robots within any planar map. 
A series of simulated case studies, comparative analysis, and real-world experiments conducted in challenging scenarios have proven the effectiveness of the proposed algorithm. 
These have been supplemented by an open-sourced implementation for the benefit of the community. 

The potentials for further development of the proposed algorithm exist in several key directions:
First, from an algorithmic perspective, the sparsity of validity checking can be developed when tether-obstacle contact points is changing during a primitive motion. 
Second, addressing the movement errors during real-world robot execution, particularly when deviations from the predefined SEF path occur, will become an urgent need. 
Online path deformation mechanism can be explored for this purpose. 
Last but not least, given the imprecise modelling of the real-world environment, implementing a maximal likelihood estimation approach to determine the state of the tether, including the last tether-obstacle contact point, would be crucial for robust SEF movement. 
\appendices

\section{Proof of Relative Angle Monotonicity (F-R)}\label{sec:appendix_SEF_FR}

The pivoting centre of the robot's circular movement is 
\begin{equation}\label{eqn:FR_pivoting_center}
\left(x_0 + R\cos(\theta_0 - \frac{\pi}{2}), y_0 + R\sin(\theta_0 - \frac{\pi}{2})\right)
\end{equation}
The robot path is parameterised as 
\begin{equation}
\alpha(t) = (x(t), y(t), \theta(t)),\ t\in (0, t_{\rm max})
\end{equation}
where $x(t)$, $y(t)$, and $\theta(t)$ are calculated as follows: 
\begin{align}\label{eqn:xytheta}
x(t) &= x_0 + R\cos(\theta_0 - \frac{\pi}{2}) - R\cos(\theta_0 - \frac{\pi}{2}-t)\\
&=x_0 +R\sin\theta_0-R\sin(\theta_0-t)\notag \\
y(t) &= y_0 + R\sin(\theta_0 - \frac{\pi}{2}) - R\sin(\theta_0 - \frac{\pi}{2}-t)\\
&=y_0-R\cos\theta_0 + R\cos(\theta_0-t)\notag \\
\theta(t) &= \theta_0 - t
\end{align}
The path of $s$ can be calculated as 
\begin{align}
\tilde{x}(t) =& x(t)  + \cos(\theta_0-t)\Delta x - \sin(\theta_0-t)\Delta y\\
\tilde{y}(t) =& y(t) + \sin(\theta_0-t)\Delta x + \cos(\theta_0-t)\Delta y
\end{align}
The relative angle function is expressed as follows
\begin{equation}
\Phi(t) = \arctan\left(\frac{o_y-\tilde{y}(t)}{o_x-\tilde{x}(t)}\right)-\theta(t),\ t\in (0, t_{\rm max})
\end{equation}
Calculating the derivative of $\tilde{x}$, $\tilde{y}$, and $\theta$ with respect to $t$ gives
\begin{align}
&\frac{d\tilde{x}}{dt} = R\cos(\theta_0-t)+\sin(\theta_0-t)\Delta x +\cos(\theta_0-t)\Delta y\\
&\frac{d\tilde{y}}{dt} = R\sin(\theta_0-t) - \cos(\theta_0-t)\Delta x + \sin(\theta_0-t)\Delta y\\
&\frac{d\theta}{dt} = -1
\end{align}
Then, the derivative of $\Phi$ is calculated as 
\begin{equation}
\frac{d\Phi}{dt} = \frac{-\frac{d\tilde{y}}{dt}(o_x-\tilde{x})+\frac{d\tilde{x}}{dt}(o_y-\tilde{y})}{(o_y - \tilde{y})^2 + (o_x - \tilde{x})^2} -\frac{d\theta}{dt}
\end{equation}
We ignore the denominator by introducing $\frac{d\tilde{\Phi}}{dt}$:
\begin{equation}\label{eqn:FR_tilde_Phi_1}
\begin{aligned}
\frac{d\tilde{\Phi}}{dt} \triangleq& \left( (o_x-\tilde{x})^2+(o_y-\tilde{y})^2 \right)\frac{d\Phi}{dt}\\
=&-\frac{d\tilde{y}}{dt}(o_x-\tilde{x}) + \frac{d\tilde{x}}{dt}(o_y-\tilde{y}) +(o_x-\tilde{x})^2+(o_y-\tilde{y})^2\\
=& (o_x - \tilde{x}-\frac{d\tilde{y}}{dt})(o_x - \tilde{x}) + (o_y + \frac{d\tilde{x}}{dt} - \tilde{y})(o_y - \tilde{y})
\end{aligned}
\end{equation}
By calculation, 
\begin{align}
o_x - \tilde{x} - \frac{d\tilde{y}}{dt} =& o_x - x(t) - R\sin(\theta_0-t)\\
=& o_x - x_0 -R\sin\theta_0\notag\\
o_y + \frac{d\tilde{x}}{dt} - \tilde{y} =& o_y - y(t) + R\cos(\theta_0-t)\\
=& o_y - y_0 +R\cos\theta_0
\end{align}
hence
\begin{equation}\label{eqn:FR_tilde_Phi_2}
\begin{aligned}
&\cdots (\mbox{Eqn.~(\ref{eqn:FR_tilde_Phi_1}}))\\
=& (o_x - x_0 -R\sin\theta_0)\\
&\qquad*(o_x-x(t)-\cos(\theta_0-t)\Delta x + \sin(\theta_0-t)\Delta y)\\
&+ (o_y - y_0 +R\cos\theta_0)\\
&\qquad*(o_y - y(t) - \sin(\theta_0-t)\Delta x - \cos(\theta_0-t)\Delta y)
\end{aligned}
\end{equation}
By letting 
\begin{align}
&A \triangleq o_x - x_0 - R\sin\theta_0\\
&B \triangleq o_y - y_0 + R\cos\theta_0
\end{align}
the equation is further simplified to 
\begin{equation}
\begin{aligned}
&\cdots (\mbox{Eqn.~(\ref{eqn:FR_tilde_Phi_2}}))\\
=& A(A + R\sin(\theta_0-t) - \cos(\theta_0-t)\Delta x+ \sin(\theta_0-t)\Delta y)\\
&+ B(B - R\cos(\theta_0-t) - \sin(\theta_0-t)\Delta x - \cos(\theta_0-t)\Delta y)\\
=& -\cos(\theta_0-t)[A\Delta x+ BR + B\Delta y]\\
&+ \sin(\theta_0-t)[AR + A\Delta y - B\Delta x] + A^2 + B^2
\end{aligned}
\end{equation}
Finally, we have
\begin{equation}
\frac{d\tilde{\Phi}}{dt} = -\sqrt{C^2+D^2}\cos(t-\theta_0-\varphi) + A^2 + B^2
\end{equation}
where 
\begin{align}
C &\triangleq A\Delta x + BR + B\Delta y\\
D &\triangleq AR + A\Delta y - B\Delta x
\end{align}
and $\varphi$ is the angle such that 
\begin{equation}
\cos\varphi = \frac{C}{\sqrt{C^2+D^2}},\ \sin\varphi = \frac{D}{\sqrt{C^2 + D^2}}
\end{equation}
Hence, the cases when $\Phi$ is monotonic are calculated as
\begin{equation}
\begin{aligned}
&\frac{d\Phi}{dt} > 0,\ \forall t\in (0, t_{\rm max})\\
\Leftrightarrow& \frac{d\tilde{\Phi}}{dt} > 0,\ \forall t\in (0, t_{\rm max})\\
\Leftrightarrow& A^2+B^2 > \sqrt{C^2+D^2}\cos(t-\theta_0-\varphi),\ \forall t\in (0, t_{\rm max})\\
\Leftrightarrow& \frac{A^2+B^2}{\sqrt{C^2+D^2}} > \cos(t-\theta_0-\varphi),\ \forall t\in (0, t_{\rm max})
\end{aligned}
\end{equation}
or
\begin{equation}
\begin{aligned}
&\frac{d\Phi}{dt} < 0,\ \forall t\in (0, t_{\rm max})\\
\Leftrightarrow& \frac{d\tilde{\Phi}}{dt} < 0,\ \forall t\in (0, t_{\rm max}) \\
\Leftrightarrow& A^2 + B^2 < \sqrt{C^2+D^2}\cos(t-\theta_0-\varphi),\ \forall t\in (0, t_{\rm max})\\
\Leftrightarrow& \frac{A^2+B^2}{\sqrt{C^2+D^2}} < \cos(t-\theta_0-\varphi),\ \forall t\in (0, t_{\rm max})
\end{aligned}
\end{equation}

\section{Proof of Relative Angle Monotonicity (F-L)}\label{sec:appendix_SEF_FL}

The pivoting centre of the robot's circular movement is 
\begin{equation}
\left(x_0 + R\cos(\theta_0 + \frac{\pi}{2}), y_0 + R\sin(\theta_0 + \frac{\pi}{2})\right)
\end{equation}
The robot path is parameterised as 
\begin{equation}
\alpha(t) = (x(t), y(t), \theta(t)),\ t\in (0, t_{\rm max})
\end{equation}
where $x(t)$, $y(t)$, and $\theta(t)$ are calculated as follows: 
\begin{align}
x(t) &= x_0 + R\cos(\theta_0 + \frac{\pi}{2}) - R\cos(\theta_0 + \frac{\pi}{2}+t)\\
&=x_0 -R\sin\theta_0+R\sin(\theta_0+t)\notag \\
y(t) &= y_0 + R\sin(\theta_0 + \frac{\pi}{2}) - R\sin(\theta_0 + \frac{\pi}{2}+t)\\
&=y_0-R\cos\theta_0 - R\cos(\theta_0+t)\notag \\
\theta(t) &= \theta_0 + t
\end{align}
The path of $s$ can be calculated as 
\begin{align}
\tilde{x}(t) &= x(t) + \cos(\theta_0 + t)\Delta x - \sin(\theta_0+t)\Delta y\\
\tilde{y}(t) &= y(t) + \sin(\theta_0 + t)\Delta x + \cos(\theta_0+t)\Delta y
\end{align}
Then the relative angle function is expressed as follows
\begin{equation}
\Phi(t) = \arctan\left(\frac{o_y-\tilde{y}(t)}{o_x-\tilde{x}(t)}\right)-\theta(t),\ t\in (0, t_{\rm max})
\end{equation}
Calculating the derivative of $\tilde{x}$, $\tilde{y}$, and $\theta$ with respect to $t$ gives
\begin{align}
&\frac{d\tilde{x}}{dt} = R\cos(\theta_0+t) - \sin(\theta_0+t)\Delta x - \cos(\theta_0+t)\Delta y\\
&\frac{d\tilde{y}}{dt} = R\sin(\theta_0+t) + \cos(\theta_0+t)\Delta x - \sin(\theta_0+t)\Delta y\\
&\frac{d\theta}{dt} = 1
\end{align}
Then, the derivative of $\Phi$ is calculated as 
\begin{equation}
\frac{d\Phi}{dt} = \frac{-\frac{d\tilde{y}}{dt}(o_x-\tilde{x})+\frac{d\tilde{x}}{dt}(o_y-\tilde{y})}{(o_y - \tilde{y})^2 + (o_x - \tilde{x})^2} -\frac{d\theta}{dt}
\end{equation}
We ignore the denominator by introducing $\frac{d\tilde{\Phi}}{dt}$: 
\begin{equation}\label{eqn:FL_tilde_Phi_1}
\begin{aligned}
\frac{d\tilde{\Phi}}{dt} \triangleq& \left( (o_x-\tilde{x})^2+(o_y-\tilde{y})^2 \right)\frac{d\Phi}{dt}\\
=&-\frac{d\tilde{y}}{dt}(o_x-\tilde{x}) + \frac{d\tilde{x}}{dt}(o_y-\tilde{y}) - (o_x-\tilde{x})^2 - (o_y-\tilde{y})^2\\
=&(-o_x + \tilde{x} - \frac{d\tilde{y}}{dt})(o_x - \tilde{x}) + (-o_y + \frac{d\tilde{x}}{dt} + \tilde{y})(o_y - \tilde{y})
\end{aligned}
\end{equation}
By calculation, 
\begin{align}
-o_x + \tilde{x} - \frac{d\tilde{y}}{dt} =& -o_x + x(t) - R\sin(\theta_0+t)\\
=& -o_x + x_0 - R\sin\theta_0\notag\\
-o_y + \frac{d\tilde{x}}{dt} + \tilde{y} =& o_y + y(t) + R\cos(\theta_0+t)\\
=& -o_y + y_0 - R\cos\theta_0\notag
\end{align}
hence 
\begin{equation}\label{eqn:FL_tilde_Phi_2}
\begin{aligned}
&\cdots(\mbox{Eqn.~(\ref{eqn:FL_tilde_Phi_1})})\\
=&(-o_x + x_0- R\sin\theta_0)\\
&\qquad*(o_x-x(t) - \cos(\theta_0+t)\Delta x + \sin(\theta_0+t)\Delta y)\\
&+(-o_y + y_0 - R\cos\theta_0)\\
&\qquad*(o_y - y(t) - \sin(\theta_0+t)\Delta x - \cos(\theta_0+t)\Delta y)
\end{aligned}
\end{equation}
By letting 
\begin{align}
A &\triangleq o_x - x_0 + R\sin\theta_0\\
B &\triangleq o_y - y_0 + R\cos\theta_0
\end{align}
the equation is further simplified to 
\begin{equation}
\begin{aligned}
&\cdots(\mbox{Eqn.~(\ref{eqn:FL_tilde_Phi_2})})\\
=& -A(A - R\sin(\theta_0+t) - \cos(\theta_0+t)\Delta x + \sin(\theta_0+t)\Delta y)\\
&-B(B + R\cos(\theta_0+t) - \sin(\theta_0+t)\Delta x - \cos(\theta_0+t)\Delta y)\\
=& \cos(\theta_0+t)[A\Delta x - BR + B\Delta y]\\
&+ \sin(\theta_0+t)[AR-A\Delta y + B\Delta x] - A^2 - B^2
\end{aligned}
\end{equation}
Finally,we have 
\begin{equation}
\frac{d\tilde{\Phi}}{dt} = \sqrt{C^2+D^2}\cos(t+\theta_0-\varphi) - A^2- B^2
\end{equation}
where 
\begin{align}
C &\triangleq A\Delta x - BR + B\Delta y\\
D &\triangleq AR - A\Delta y + B\Delta x
\end{align}
and $\varphi$ is the angle such that 
\begin{equation}
\cos\varphi = \frac{C}{\sqrt{C^2+D^2}},\ \sin\varphi = \frac{D}{\sqrt{C^2+D^2}}
\end{equation}
Hence, the cases when $\Phi$ is monotonic are calculated as
\begin{equation}
\begin{aligned}
&\frac{d\Phi}{dt} > 0,\ \forall t\in (0, t_{\rm max})\\
\Leftrightarrow& \frac{d\tilde{\Phi}}{dt} > 0,\ \forall t\in (0, t_{\rm max})\\
\Leftrightarrow& A^2+B^2 < \sqrt{C^2+D^2}\cos(t+\theta_0-\varphi),\ \forall t\in (0, t_{\rm max})\\
\Leftrightarrow& \frac{A^2+B^2}{\sqrt{C^2+D^2}} < \cos(t+\theta_0-\varphi),\ \forall t\in (0, t_{\rm max})
\end{aligned}
\end{equation}
or 
\begin{equation}
\begin{aligned}
&\frac{d\Phi}{dt} < 0,\ \forall t\in (0, t_{\rm max})\\
\Leftrightarrow& \frac{d\tilde{\Phi}}{dt} < 0,\ \forall t\in (0, t_{\rm max}) \\
\Leftrightarrow& A^2+B^2 > \sqrt{C^2+D^2}\cos(t+\theta_0-\varphi),\ \forall t\in (0, t_{\rm max})\\
\Leftrightarrow& \frac{A^2+B^2}{\sqrt{C^2+D^2}} > \cos(t+\theta_0-\varphi),\ \forall t\in (0, t_{\rm max})
\end{aligned}
\end{equation}

\section{Proof of Tether Length Monotonicity (Straight)}\label{sec:appendix_TLA_Straight}
The robot path is parameterised as 
\begin{align}
x(t) &= x_0 + t\cos\theta_0\\
y(t) &= y_0 + t\sin\theta_0\\
\theta(t) &= \theta_0
\end{align}
Then, the path of $s$ is 
\begin{align}
\tilde{x}(t) =& x_0 + t\cos\theta_0 + \cos\theta_0\Delta x - \sin\theta_0\Delta y\\
\tilde{y}(t) =& y_0 + t\sin\theta_0 + \sin\theta_0\Delta x + \cos\theta_0\Delta y
\end{align}
The tether length function is expressed as follows
\begin{equation}
L(t) = L_o + \sqrt{(o_x - \tilde{x}(t))^2 + (o_y - \tilde{y}(t))^2}
\end{equation}
where $L_o$ is the consumed tether length from the base point to the last tether-obstacle contact point $o$. 
Calculating the derivative of $\tilde{x}$, $\tilde{y}$, and $\theta$ with respect to $t$ gives
\begin{align}
\frac{d\tilde{x}}{dt} &= \cos\theta_0\\
\frac{d\tilde{y}}{dt} &= \sin\theta_0\\
\frac{d\theta}{dt} &= 0
\end{align}
After that, the derivative of $L$ with respect to $t$ is calculated as 
\begin{equation}
\frac{dL}{dt} = \frac{-(o_x-\tilde{x})\frac{d\tilde{x}}{dt} - (o_y - \tilde{y})\frac{d\tilde{y}}{dt}}{\sqrt{(o_x-\tilde{x})^2 + (o_y - \tilde{y})^2}}
\end{equation}
We ignore the denominator by introducing $\frac{d\tilde{L}}{dt}$: 
\begin{equation}
\begin{aligned}
\frac{d\tilde{L}}{dt} \triangleq& \sqrt{(o_x-\tilde{x})^2 + (o_y - \tilde{y})^2} \frac{dL}{dt}\\
=& -(o_x - x_0 - t\cos\theta_0-\cos\theta_0\Delta x + \sin\theta_0\Delta y)\cos\theta_0\\
&- (o_y - y_0 - t\sin\theta_0 - \sin\theta_0\Delta x - \cos\theta_0\Delta y)\sin\theta_0\\
=& t - (o_x + o_y - x_0\cos\theta_0 - y_0\sin\theta_0 - \Delta x)
\end{aligned}
\end{equation}
Hence, the cases when $L$ is monotonic are calculated as
\begin{equation}
\begin{aligned}
&\frac{dL}{dt} >0,\ \forall t\in (0, t_{\rm max})\\
\Leftrightarrow& \frac{d\tilde{L}}{dt}>0,\ \forall t\in (0, t_{\rm max})\\
\Leftrightarrow& 0 > o_x + o_y - x_0\cos\theta_0 - y_0\sin\theta_0 - \Delta x
\end{aligned}
\end{equation}
or 
\begin{equation}
\begin{aligned}
&\frac{dL}{dt} < 0,\ \forall t\in (0, t_{\rm max})\\
\Leftrightarrow& \frac{d\tilde{L}}{dt} < 0,\ \forall t\in (0, t_{\rm max})\\
\Leftrightarrow& t_{\rm max} < o_x + o_y - x_0\cos\theta_0 - y_0\sin\theta_0 - \Delta x
\end{aligned}
\end{equation}

\section{Proof of Tether Length Monotonicity (F-R)}\label{sec:appendix_TLA_FR}
The pivoting centre of the robot's circular movement is 
\begin{equation}
\left( x_0 + R\cos(\theta_0 - \frac{\pi}{2}),\ y_0 + R\sin(\theta_0 - \frac{\pi}{2}) \right)
\end{equation}
The robot path is parameterised as 
\begin{equation}
\alpha(t)= (x(t), y(t), \theta(t)), \forall t\in (0, t_{\rm max})
\end{equation}
where $x(t)$, $y(t)$, and $\theta(t)$ are calculated as follows: 
\begin{align}
x(t) &= x_0 + R\cos(\theta_0 - \frac{\pi}{2}) - R\cos(\theta_0 - \frac{\pi}{2}-t)\\
&=x_0 +R\sin\theta_0-R\sin(\theta_0-t)\notag \\
y(t) &= y_0 + R\sin(\theta_0 - \frac{\pi}{2}) - R\sin(\theta_0 - \frac{\pi}{2}-t)\\
&=y_0-R\cos\theta_0 + R\cos(\theta_0-t)\notag \\
\theta(t) &= \theta_0 - t
\end{align} 
The path of $s$ can be calculated as
\begin{align}
\tilde{x}(t) =& x(t) + \cos(\theta_0-t)\Delta x - \sin(\theta_0-t)\Delta y\\
=& x_0 + R\sin\theta_0 - (R+\Delta y)\sin(\theta_0-t) + \cos(\theta_0-t)\Delta x  \notag\\
\tilde{y}(t) =& y(t) + \sin(\theta_0-t)\Delta x + \cos(\theta_0-t)\Delta y\\
=& y_0 - R\cos\theta_0 + (R+\Delta y)\cos(\theta_0-t) + \sin(\theta_0-t)\Delta x\notag
\end{align}
The tether length function is expressed as follows
\begin{equation}
L(t) = L_o + \sqrt{ (o_x-\tilde{x}(t))^2 + (o_y - \tilde{y}(t))^2}  
\end{equation}
where $L_o$ is the consumed tether length from the base point to the last tether-obstacle contact point $o$. 
Calculating the derivative of $\tilde{x}$, $\tilde{y}$, and $\theta$ with respect to $t$ gives
\begin{align}
\frac{d\tilde{x}}{dt} =& (R+\Delta y) \cos(\theta_0-t) + \sin(\theta_0-t)\Delta x\\
\frac{d\tilde{y}}{dt} =& (R+\Delta y)\sin(\theta_0-t) - \cos(\theta_0-t)\Delta x\\
\frac{d\theta}{dt} =& -1
\end{align}
The derivative of $L$ is calculated as
\begin{equation}
\frac{dL}{dt} = \frac{-(o_x-\tilde{x})\frac{d\tilde{x}}{dt} - (o_y - \tilde{y})\frac{d\tilde{y}}{dt}}{\sqrt{(o_x-\tilde{x})^2 + (o_y - \tilde{y})^2}}
\end{equation}
We ignore the denominator by introducing $\frac{d\tilde{L}}{dt}$: 
\begin{equation}\label{eqn:FR_tilde_L_1}
\begin{aligned}
\frac{d\tilde{L}}{dt} \triangleq& \sqrt{(o_x - \tilde{x})^2+(o_y - \tilde{y})^2}\frac{dL}{dt}\\
=& -(o_x - x_0 - R\sin\theta_0 + (R+\Delta y)\sin(\theta_0-t) - \cos(\theta_0-t)\Delta x)\\
&*((R+\Delta y)\cos(\theta_0-t) + \sin(\theta_0-t)\Delta x)\\
&-(o_y - y_0 + R\sin\theta_0 - (R+\Delta y)\cos(\theta_0-t) - \sin(\theta_0-t)\Delta x)\\
&*((R+\Delta y)\sin(\theta_0-t) - \cos(\theta_0-t)\Delta x)
\end{aligned}
\end{equation}
By letting 
\begin{align}
A\triangleq& o_x - x_0 - R\sin\theta_0\\
B\triangleq& o_y - y_0 + R\sin\theta_0
\end{align}
the equation is further simplified to 
\begin{equation}
\begin{aligned}
&\cdots(\mbox{Eqn.~(\ref{eqn:FR_tilde_L_1})})\\
=& \sin(\theta_0-t)\cos(\theta_0-t)[-(R+\Delta y)^2 + \Delta x^2 + (R+\Delta y)^2 - \Delta x^2]\\
&+ \sin^2(\theta_0-t)[-(R+\Delta y)\Delta x + \Delta x(R + \Delta y)]\\
&+\cos^2(\theta_0-t)[\Delta x(R + \Delta y) - (R+\Delta y)\Delta x]\\
&-A((R+\Delta y)\cos(\theta_0-t)+\sin(\theta_0-t)\Delta x))\\
&-B((R + \Delta y)\sin(\theta_0-t) - \cos(\theta_0-t)\Delta x)\\
=& (B\Delta x - AR - A\Delta y)\cos(\theta_0-t)\\ 
&\qquad- (A\Delta x - BR - B\Delta y)\sin(\theta_0-t)
\end{aligned}
\end{equation}
Further simplify the notations, we have 
\begin{equation}
\frac{d\tilde{L}}{dt} = \sqrt{C^2 + D^2}\cos(t-\theta_0 - \varphi)
\end{equation}
where 
\begin{align}
&C \triangleq B\Delta x - AR - A\Delta y\\
&D \triangleq A\Delta x - BR - B\Delta y
\end{align}
and $\varphi$ is the angle such that 
\begin{equation}
\cos\varphi = \frac{C}{\sqrt{C^2+D^2}},\ \sin\varphi = \frac{D}{\sqrt{C^2+D^2}}
\end{equation}
Hence, the cases when $L$ is monotonic are calculated as 
\begin{equation}
\begin{aligned}
&\frac{dL}{dt} > 0,\ \forall t\in (0, t_{\rm max})\\
\Leftrightarrow&\frac{d\tilde{L}}{dt} > 0,\ \forall t\in (0, t_{\rm max})\\
\Leftrightarrow& \sqrt{C^2+D^2}\cos(t - \theta_0 - \varphi) > 0,\ \forall t\in (0, t_{\rm max})\\
\Leftrightarrow& \cos(t - \theta_0 - \varphi) > 0,\ \forall t\in (0, t_{\rm max})
\end{aligned}
\end{equation}
or 
\begin{equation}
\begin{aligned}
&\frac{dL}{dt} < 0,\ \forall t\in (0, t_{\rm max})\\
\Leftrightarrow& \frac{d\tilde{L}}{dt} < 0,\ \forall t\in (0, t_{\rm max})\\
\Leftrightarrow& \sqrt{C^2+D^2}\cos(t - \theta_0 - \varphi) < 0,\ \forall t\in (0, t_{\rm max})\\
\Leftrightarrow& \cos(t - \theta_0 - \varphi) < 0,\ \forall t\in (0, t_{\rm max})
\end{aligned}
\end{equation}

\section{Proof of Tether Length Monotonicity (F-L)}\label{sec:appendix_TLA_FL}

The pivoting centre of the robot's circular movement is 
\begin{equation}
\left(x_0 + R\cos(\theta_0 + \frac{\pi}{2}),\ y_0 + R\sin(\theta_0 + \frac{\pi}{2})\right)
\end{equation}
The robot path is parameterised as 
\begin{equation}
\alpha(t) = (x(t), y(t), \theta(t)),\ t\in (0, t_{\rm max})
\end{equation}
where $x(t)$, $y(t)$, and $\theta(t)$ are calculated as follows: 
\begin{align}
x(t) &= x_0 + R\cos(\theta_0 + \frac{\pi}{2}) - R\cos(\theta_0 + \frac{\pi}{2}+t)\\
&=x_0 -R\sin\theta_0+R\sin(\theta_0+t)\notag \\
y(t) &= y_0 + R\sin(\theta_0 + \frac{\pi}{2}) - R\sin(\theta_0 + \frac{\pi}{2}+t)\\
&=y_0-R\cos\theta_0 - R\cos(\theta_0+t)\notag \\
\theta(t) &= \theta_0 + t
\end{align}
The path of $s$ can be calculated as
\begin{align}
\tilde{x}(t) =&x(t) + \cos(\theta_0+t)\Delta x - \sin(\theta_0+t)\Delta y\\
=& x_0-R\sin\theta_0 + (R - \Delta y)\sin(\theta_0+t) + \cos(\theta_0+t)\Delta x\notag\\
\tilde{y}(t) =& y(t) + \sin(\theta_0+t)\Delta x + \cos(\theta_0+t)\Delta y\\
=& y_0 - R\cos\theta_0 - (R-\Delta y)\cos(\theta_0+t) + \sin(\theta_0+t)\Delta x \notag
\end{align}
The tether length function is expressed as follows
\begin{equation}
L(t) = L_o + \sqrt{ (o_x-\tilde{x}(t))^2 + (o_y - \tilde{y}(t))^2}  
\end{equation}
where $L_o$ is the consumed tether length from the base point to the last tether-obstacle contact point $o$. 
Calculating the derivative of $\tilde{x}$, $\tilde{y}$, and $\theta$ with respect to $t$ gives
\begin{align}
&\frac{d\tilde{x}}{dt} = (R-\Delta y)\cos(\theta_0+t) - \sin(\theta_0+t)\Delta x\\
&\frac{d\tilde{y}}{dt} = (R-\Delta y)\sin(\theta_0+t) +\cos(\theta_0+t)\Delta x\\
&\frac{d\theta}{dt} = 1
\end{align}
The derivative of $L$ is calculated as
\begin{equation}
\frac{dL}{dt} = \frac{-(o_x-\tilde{x})\frac{d\tilde{x}}{dt} - (o_y - \tilde{y})\frac{d\tilde{y}}{dt}}{\sqrt{(o_x-\tilde{x})^2 + (o_y - \tilde{y})^2}}
\end{equation}
We ignore the denominator by introducing $\frac{d\tilde{L}}{dt}$: 
\begin{equation}\label{eqn:FL_tilde_L_1}
\begin{aligned}
\frac{d\tilde{L}}{dt} \triangleq& \sqrt{(o_x-\tilde{x})^2 + (o_y - \tilde{y})^2} \frac{dL}{dt}\\
=& -(o_x - x_0 + R\sin\theta_0 - (R-\Delta y)\sin(\theta_0+t) - \cos(\theta_0+t)\Delta x)\\
&*((R - \Delta y)\cos(\theta_0+t) - \sin(\theta_0+t)\Delta x)\\
&-(o_y - y_0 + R\cos\theta_0 + (R-\Delta y)\cos(\theta_0+t) - \sin(\theta_0+t)\Delta x)\\
&*((R - \Delta y)\sin(\theta_0+t) +\cos(\theta_0+t)\Delta x)
\end{aligned}
\end{equation}
By letting 
\begin{align}
A \triangleq& o_x - x_0 + R\sin\theta_0\\
B \triangleq& o_y - y_0 + R\cos\theta_0
\end{align}
the equation is further simplified to 
\begin{equation}
\begin{aligned}
&\cdots(\mbox{Eqn.~(\ref{eqn:FL_tilde_L_1})})\\
=& \sin(\theta_0+t)\cos(\theta_0+t)[(R-\Delta y)^2 - \Delta x^2 - (R - \Delta y)^2 + \Delta x^2]\\
&+\sin^2(\theta_0+t)[-(R-\Delta y)\Delta x + \Delta x(R-\Delta y)]\\
&+\cos^2(\theta_0+t)[\Delta x(R-\Delta y) - (R-\Delta y)\Delta x]\\
&- A((R-\Delta y)\cos(\theta_0+t) - \sin(\theta_0+t)\Delta x)\\
&- B((R-\Delta y)\sin(\theta_0+t) + \cos(\theta_0+t)\Delta x)\\
=& (A\Delta y - AR - B\Delta x)\cos(\theta_0+t)\\
&\qquad - (A\Delta x - BR + B\Delta y)\sin(\theta_0+t)
\end{aligned}    
\end{equation}
Further simplify the notations, we have 
\begin{equation}
\frac{d\tilde{L}}{dt}= \sqrt{C^2 + D^2}\cos(t + \theta_0 + \varphi)
\end{equation}
where
\begin{align}
C &\triangleq A\Delta y - AR - B\Delta x\\
D &\triangleq A\Delta x - BR + B\Delta y
\end{align}
and $\varphi$ is the angle such that 
\begin{equation}
\cos\varphi = \frac{C}{\sqrt{C^2+D^2}},\ \sin\varphi = \frac{D}{\sqrt{C^2+D^2}}
\end{equation}
Hence, the cases when $L$ is monotonic are calculated as 
\begin{equation}
\begin{aligned}
&\frac{dL}{dt} > 0,\ \forall t\in (0, t_{\rm max})\\
\Leftrightarrow&\frac{d\tilde{L}}{dt} > 0,\ \forall t\in (0, t_{\rm max})\\
\Leftrightarrow& \sqrt{C^2+D^2}\cos(t + \theta_0 + \varphi)>0,\ \forall t\in (0, t_{\rm max})\\
\Leftrightarrow& \cos(t + \theta_0 + \varphi)>0,\ \forall t\in (0, t_{\rm max})
\end{aligned}
\end{equation}
or 
\begin{equation}
\begin{aligned}
&\frac{dL}{dt} < 0,\ \forall t\in (0, t_{\rm max})\\
\Leftrightarrow&\frac{d\tilde{L}}{dt} < 0,\ \forall t\in (0, t_{\rm max})\\
\Leftrightarrow& \sqrt{C^2+D^2}\cos(t + \theta_0 + \varphi)<0,\ \forall t\in (0, t_{\rm max})\\
\Leftrightarrow& \cos(t + \theta_0 + \varphi)<0,\ \forall t\in (0, t_{\rm max})
\end{aligned}
\end{equation}

\bibliographystyle{ieeetr}
\bibliography{tro23}

\end{document}